\documentclass{article}

\PassOptionsToPackage{numbers}{natbib}

    \usepackage[final]{neurips_2025}

\usepackage[utf8]{inputenc} %
\usepackage[T1]{fontenc}    %
\usepackage{hyperref}       %
\usepackage{url}            %
\usepackage{booktabs}       %
\usepackage{amsfonts}       %
\usepackage{nicefrac}       %
\usepackage{microtype}      %
\usepackage{xcolor}         %

\usepackage{amsmath,cleveref}
\usepackage{hyperref}
\newcommand{\todo}[1]{}

\usepackage{amsmath}
\usepackage{amssymb}
\usepackage{amsfonts}
\usepackage{amsthm}
\usepackage{dsfont}
\usepackage{graphicx}
\usepackage{enumitem}

\hypersetup{
    colorlinks,
    linkcolor={red!50!black},
    citecolor={blue!50!black},
    urlcolor={blue!80!black},
    pdftitle={Evolutionary Prediction Games}
}

\usepackage{apptools}
\AtAppendix{\counterwithin{proposition}{section}}
\AtAppendix{\counterwithin{lemma}{section}}
\AtAppendix{\counterwithin{theorem}{section}}
\AtAppendix{\counterwithin{assumption}{section}}
\AtAppendix{\counterwithin{definition}{section}}
\AtAppendix{\counterwithin{remark}{section}}

\ifx\draftmode\undefined
\else

\fi
\newcommand{\extended}[1]{} %
\renewcommand{\todo}[1]{}

\newcommand{\tocitec}[1]{}

\newcommand{\torefc}[1]{}
\newcommand{\nir}[1]{}
\newcommand{\eden}[1]{}

\newcommand{\sectionoutline}[1]{}
\renewcommand{\sectionoutline}[1]{}

\newcommand{\es}[1]{}

\newcommand{\squeeze}{\looseness=-1}

\newcommand{\PartialDerivative}[2]{\frac{\partial{#1}}{\partial{#2}}}

\newcommand{\Norm}[1]{\left\lVert{#1}\right\rVert}
\newcommand{\Size}[1]{{\left|{#1}\right|}}

\newcommand{\Set}[1]{{\left\{{#1}\right\}}}
\newcommand{\Reals}{\mathbb{R}}
\newcommand{\NonNegativeReals}{\Reals_{\ge 0}}

\newcommand{\UnitInterval}{\left[0,1\right]}
\newcommand{\Indicator}[1]{\mathds{1}\left\{{#1}\right\}}
\DeclareMathOperator*{\sgn}{sgn}
\DeclareMathOperator*{\support}{support}
\newcommand\expect[2]{\mathbb{E}_{#1}{\left[ {#2} \right]}}
\newcommand\prob[2]{\mathbb{P}_{#1}{\left[ {#2} \right]}}

\newcommand{\Features}{\mathcal{X}}
\newcommand{\Labels}{\mathcal{Y}}
\newcommand{\Loss}{\mathcal{L}}
\newcommand{\LossVec}{\vec{\mathcal{L}}}

\newcommand{\DistOver}[1]{{\Delta\left( {#1} \right)}}
\newcommand{\DistOverGroups}{\Delta^K}
\newcommand{\SingletonDistribution}{\mathbf{e}}
\newcommand{\acc}{\mathrm{acc}}

\newcommand{\Hypotheses}{\mathcal{H}}
\newcommand{\Learner}{\mathcal{A}}

\newcommand{\FeaturesLabels}{{\Features\times\Labels}}

\newcommand{\Retention}{\nu}

\newcommand{\yhat}{\hat{y}}
\newcommand{\opt}{\mathrm{opt}}
\newcommand{\const}{\mathrm{const}}
\newcommand{\otherwise}{\mathrm{otherwise}}
\newcommand{\tdt}{\tfrac{\mathrm d}{\mathrm d t}}
\newcommand{\fdt}{\frac{\mathrm d}{\mathrm d t}}

\newtheorem{theorem}{Theorem}
\newtheorem{lemma}{Lemma}

\newtheorem{definition}{Definition}

\newtheorem{proposition}{Proposition}

\newtheorem{remark}{Remark}

\newcommand{\appendixref}[1]{\hyperref[#1]{appendix~\ref*{#1}}}
\newcommand{\Appendixref}[1]{\hyperref[#1]{Appendix~\ref*{#1}}}

\newenvironment{itemizecompact}{\begin{itemize}[leftmargin=1em,topsep=0em,itemsep=0.2em]}{\end{itemize}}
\newenvironment{enumeratecompact}{\begin{enumerate}[leftmargin=1.7em,topsep=0em,itemsep=0.1em]}{\end{enumerate}}

\newcommand{\nnalg}{1\nobreakdash-NN}

\usepackage{amsmath,amsfonts,bm}

\def\eqref#1{equation~\ref{#1}}

\def\1{\bm{1}}

\def\vzero{{\bm{0}}}
\def\vone{{\bm{1}}}

\def\vb{{\bm{b}}}

\def\vp{{\bm{p}}}
\def\vq{{\bm{q}}}

\def\vv{{\bm{v}}}

\def\mI{{\bm{I}}}

\def\mPhi{{\bm{\Phi}}}

\DeclareMathAlphabet{\mathsfit}{\encodingdefault}{\sfdefault}{m}{sl}
\SetMathAlphabet{\mathsfit}{bold}{\encodingdefault}{\sfdefault}{bx}{n}

\DeclareMathOperator*{\argmax}{arg\,max}
\DeclareMathOperator*{\argmin}{arg\,min}

\DeclareMathOperator{\sign}{sign}

\newcommand{\ParamCoexistenceSvmRegularizationExperimentAlpha}{0.75}
\newcommand{\ParamEmpiricalCifarAugmentationCoexistenceAccuracyConfidence}{93.5 \pm 0.1\%}

\newcommand{\ParamEmpiricalCifarAugmentationNRepetitions}{20}
\newcommand{\ParamEmpiricalCifarAugmentationSingleGroupAccuracyConfidence}{92.6 \pm 0.1\%}

\newcommand{\ParamEmpiricalMnistLabelNoiseCoexistenceAccuracyConfidence}{80.4 \pm 0.2\%}

\newcommand{\ParamEmpiricalMnistLabelNoiseNRepetitions}{50}
\newcommand{\ParamEmpiricalMnistLabelNoiseTrainAccuracyConfidence}{98.7 \pm 0.1\%}

\newcommand{\ParamFolktablesGridNumRepetitions}{10}

\newcommand{\ParamFolktablesStateCaliforniaNumDatapoints}{195,665}
\newcommand{\ParamFolktablesStateNewYorkNumDatapoints}{103,021}
\newcommand{\ParamFolktablesStateTexasNumDatapoints}{135,924}
\newcommand{\ParamFolktablesSvmDisparityTime}{316}
\newcommand{\ParamFolktablesTestSetSize}{5,000}
\newcommand{\ParamFolktablesTrainingSetSize}{1000}

\newcommand{\ParamSensitivityGaussianProcessAlpha}{4}
\newcommand{\ParamSensitivityGaussianProcessLengthScale}{0.1}
\newcommand{\ParamSensitivityInitialStateRangeMax}{0.65}
\newcommand{\ParamSensitivityInitialStateRangeMin}{0.35}

\newcommand{\ParamSensitivityTimeToDominanceThreshold}{0.01}

\title{Evolutionary Prediction Games}

\newcommand{\Affiliation}[2]{
{#1}\\
Technion -- Israel Institute of Technology\\
Haifa, Israel\\
\texttt{{#2}@cs.technion.ac.il}
}
\author{%
    \Affiliation{Eden Saig}{edens}
    \And
    \Affiliation{Nir Rosenfeld}{nirr}
}

\begin{document}

\maketitle

\begin{abstract}

When a prediction algorithm serves
a collection of users,
disparities in prediction quality are likely to emerge.
If users respond to accurate predictions by increasing engagement,
inviting friends,
or adopting trends,
repeated learning creates a feedback loop that shapes both the model and the population of its users.
In this work, we introduce \emph{evolutionary prediction games},
a framework grounded in evolutionary game theory which models such feedback loops as natural-selection processes among groups of users. Our theoretical analysis reveals a gap between idealized and real-world learning settings: 
In idealized settings with unlimited data and 
computational power,
repeated learning 
creates competition 
and
promotes competitive exclusion 
across a broad class of behavioral dynamics.
However, under realistic constraints such as finite data, limited compute, or risk of overfitting,
we show that stable coexistence and mutualistic symbiosis between groups becomes possible.
We analyze these possibilities in terms of their stability and feasibility,
present mechanisms that can sustain their existence,
and empirically demonstrate our findings.
\squeeze

\end{abstract}

\section{Introduction} \label{sec:intro}
Accurate predictions have become essential for any platform that supports user decision-making. Improvements in prediction accuracy often directly translate to better quality of service with benefits to both the platform and its users.
This is one reason why modern platforms ranging from
content recommendation 
and online marketplaces
to personalized education and medical services
have come to rely on machine learning as their backbone,
and are investing much effort and resources in continually improving their predictions. However, while generally beneficial, promoting accuracy blindly can have unexpected,
and in some cases undesired, consequences \citep{chaney2018algorithmic, tsipras2018robustness,hardt2023performative}. It is therefore important to understand the possible downstream and long-term effects of learning on 
social outcomes.

Conventional learning approaches aim to maximize accuracy
on a given, predetermined data distribution.
But in social settings,
the distributions are composed of those users who \emph{choose} to use the platform.
When such choices depend on the quality of predictions,
learning becomes a driver of user self-selection,
and thus gains influence over its user population.
This creates a feedback loop:
model deployment shifts the population,
and population changes trigger model retraining.
We are interested in understanding the general tendencies and possible long-term outcomes of this process.
\squeeze

In particular, our focus is on user feedback dynamics in which accurate predictions encourage
engagement or adoption,
and prediction errors discourage them.
Such dynamics arise across domains: accurate recommendations drive network-effect growth as users invite peers \citep{leskovec2007dynamics}; 
successful content strategies on social media are mimicked by others \citep{bian2023influencer}; 
precise credit-risk models reduce premiums and therefore attract more clients to loan programs \citep{khandani2010consumer}; 
and medical providers with higher diagnostic accuracy draw more patients \citep{bertrand2022patient}. 
Common to the above is that
individual choices adhere to some form of group structure in the population, and that in aggregate, these groups tend to grow when prediction quality is higher. 
Note that this notion of ``group'' is flexible,
as groups can represent different demographics, behaviors, or roles,
and with memberships being either inherent or chosen.

Focusing on user choices as the primary driving force,
we seek to analyze the impact of learning on group proportions over time.
For this purpose,
we adopt a novel evolutionary perspective and model the 
joint dynamics of learning and user choices
through the lens of \emph{natural selection}.
Under the assertion that some degree of predictive error always exists,
our key modeling point is that accuracy becomes, in effect, a \emph{scarce resource}
over which different groups in the population ``compete''. %
By associating each group's accuracy with its evolutionary fitness,
we obtain an evolutionary process in which learning is the
driver of selective pressure,
and consequently, a determinant of long-term population outcomes.
We can then ask questions regarding the long-term tendencies of the population composition,
overall and per-group accuracies,
and the affect of different modeling and algorithmic choices on
temporal trends and long-term evolutionary outcomes.

\paragraph{Contributions.} 
Our main conceptual innovation is to analyze these feedback loops using a \emph{population game}---%
a core component of evolutionary game theory useful for studying the dynamics of agents driven by local interaction rules
\citep{ smith1982evolution, hofbauer1998evolutionary, sandholm2010population}. 
Evolutionary game theory relates between game-theoretic properties of population games (e.g., Nash equilibria), and the dynamics of a wide variety of natural selection dynamics, such as imitation, word-of-mouth influence, or social learning,
all of which are supported by our analysis.
This makes the framework well-suited for our setting,
as it allows us to address the common thread in the wide variety of feedback loops described above in a unified way.
\squeeze

Towards this,
we first propose a novel game, called an \emph{evolutionary prediction game},
in which evolutionary fitness is associated with prediction accuracy.
This allows us to
study the co-evolution of a (re)trained model and the population of its users.
We then analyze the structure of games induced by different learning algorithms and their corresponding dynamics,
and characterize the possible outcomes under different settings.
Our results include
conditions under which `survival of the fittest' is a likely outcome;
mechanisms that nonetheless enable coexistence;
a discussion of the role and effects of retraining;
and connections between survival, accuracy, and (evolutionary) stability.
We also discuss connections to fairness and social conservation.
\squeeze

We conclude by complementing our analysis
with experiments using both synthetic and real data and coupled with simulated dynamics.
Our empirical results shed light on when and how certain user groups are likely to dominate, disappear, or coexist,
and demonstrate how different design choices can shape social outcomes---even if inadvertently.
Together, these highlight the importance of understanding, anticipating, and accounting for the long-term effects of learning on its user population.
\squeeze

\if
Accurate predictions
have potential to improve many aspects of our lives,
whether as individuals, as groups, or as a society.
From everyday 
content recommendation to 
life-changing medical diagnoses,
higher quality predictions in human-facing tasks often lead to better decisions and therefore better outcomes
for users.
A reasonable conclusion is that we should work hard to further push the envelope on predictive accuracy across applications.
Considerable efforts in machine learning research and practice are devoted precisely to this purpose.
But at the same time, there is a growing recognition that blindly %
promoting accuracy in social contexts can have unexpected,
and sometimes undesired, consequences.
It is therefore important to consider not only \emph{how much},
but also \emph{how},
predictive accuracy is attained.
\squeeze

Concerns regarding the application of machine learning in society typically
consider its possible ill effects 
on a \emph{given} population of users.
Taking this notion one step further, here we argue that learning can have the potency to \emph{determine} what this population is---or will be. 
Our main observation is that when the benefit to users depends on the accuracy of the predictions they receive,
then consequently, so will the willingness of those users to use a given
platform or service. %
Since users' decisions of whether or not to use a system
are in aggregate and over time 
precisely what forms the population of system users,
learning gains influence over
the population's eventual composition.
Our goal is to study the capacity of learning to do so---%
a point we believe is often overlooked, but carries significant implications.

Our working hypothesis
is that not all users can be equally happy;
or more technically,
that in typical settings,
not all users in a population will get equally accurate predictions.
If we make the plausible assumption that users who get higher accuracy are more likely to stay (or even invite their friends),
whereas users who get lower accuracy are more prone to leave
(or convince others to leave as well),
then over time, learning will come to shape the population
and the relative sizes of different user groups.
From 
a learning perspective,
once the population changes,
it becomes beneficial to retrain the model on new data from the updated distribution.
This creates a feedback loop in which changes to the model induce changes in the population, which then lead to further changes to the model, ad infinitum. %
We are interested in understanding the long-term outcomes of this process,
both in terms of the attainable accuracy %
and the resulting composition of the population.

Towards this, we adopt a novel evolutionary perspective and model the 
joint dynamics of learning and user choices
through the lens of \emph{natural selection}.
Under the assertion that some degree of predictive error always exists,
our key modeling point is that accuracy becomes, in effect, a scarce resource
over which different groups in the population `compete'. %
By associating each group's accuracy with its evolutionary fitness,
we obtain an evolutionary process in which learning is the
driver of selective pressure,
and consequently, a determinant of long-term population outcomes.

Our main innovation is to model this process using a \emph{population game},
which we adopt from the field of evolutionary game theory \citep[e.g.,][]{hofbauer1998evolutionary},
which studies the dynamics of populations of agents driven by local, myopic interaction rules.
Population games are a core component of this framework,
defining the evolutionary fitness of each group 
as a function of the overall population state.
Formalizing evolution as a game connects key properties of natural selection processes (e.g. stationary points)
with properties of the corresponding population game (e.g. Nash equilibria).
Such links have previously been established for a large family of games and dynamics \citep[see, e.g., ][]{sandholm2010population},
and successfully used for explaining a wide array of
evolutionary phenomena from survival of the fittest and kin selection
to altruistic behavior and mating rituals. %

Population games are %
useful for studying the limiting behavior of global population dynamics driven by local 
interactions,
such as imitation, word-of-mouth, %
or social learning.
This makes them well-suited for adaptation to our setting.
Building on this as motivation, we propose a new form of population game,
which we call an \emph{evolutionary prediction game},
that models the coevolution of a (re)trained model and the population of users over which it makes predictions.
Evolutionary prediction game are unique in that the interaction between the different groups are induced by solutions to a non-trivial optimization problem, namely loss minimization.
This complexity introduces inherent challenges to characterizing the game's outcomes,
but at the same time, provides useful structure, which we exploit in our analysis.

A primary question that evolutionary game theory aims to answer is: which species survive?
Typically, there are two possible outcomes of interest:
survival of some species (and extinction of others), or coexistence.
Evolutionary prediction games allow us to explore which of these outcomes materialize, 
and when,
for a population of social groups in settings where learning acts as a selective force.
In nature, most interactions result in the `survival of the fittest',
an outcome which is so pronounced that it is often associated with the mechanism of natural selection itself.
Our first result establishes that this occurs in learning as well,
in two settings: (i) when the model is trained only once at the onset (i.e., is not updated over time), and
(ii) under the ideal conditions of infinite data and unlimited compute.
While intuitive, proving this for evolutionary prediction games is technically challenging due to the complexity introduced by the loss minimization operator.
Our main tool here is to show that evolutionary prediction games in such settings adhere to the structure of more general \emph{potential games} \citep{sandholm2001potential}.
This allows us to connect evolutionary outcomes to extrema of scalar-valued functions over the simplex, providing interpretation and simplifying analysis.

One interesting point that our analysis reveals is that,
in the idealized setting, although survival of the fittest is the likely outcome,
coexistence becomes \emph{a possibility} under retraining.
Technically, this manifests as the (possible) existence of a mixed equilibrium,
but which is unstable (and so will not materialize under plausible dynamics).
Our next set of results then show that \emph{learning under realistic conditions},
i.e., using finite data and/or limited compute,
can stabilize such mixed equilibria,
meaning that coexistence becomes possible (or even probable).
We demonstrate several mechanisms through which this phenomenon can arise, all common in machine learning practice, namely:
(i) the use of proxy losses (e.g., hinge loss or cross entropy),
(ii) the need to overcome overfitting, 
and (iii) asymmetric label noise.
We conduct experiments using synthetic and real data:
the former to provide insight into the inner workings of these mechanisms,
and the latter to demonstrate their capacity for coexistence under more realistic conditions and at scale.%
\fi

\subsection{Related Work}

\paragraph{Evolutionary game theory.}
Evolutionary game theory was originally developed as a framework for modeling natural selection processes in biology \citep{moran1958random,smith1973logic, smith1982evolution, hofbauer1998evolutionary, sigmund1999evolutionary},
and has since been applied to the analysis of large populations of myopic interacting agents in economics \citep{sandholm2010population, fudenberg1998theory}, 
and to the study of online learning algorithms \citep{freund1999adaptive, cesa2006prediction}. 
From a technical standpoint, the games we propose and analyze
are unique in this space as they are
defined implicitly as solutions to statistical optimization problems. Our results for oracle classifiers reflect the principle of competitive exclusion \citep{hardin1960competitive}, while our analysis of non-optimal classifiers relates to long-studied questions of coexistence \citep{chesson2000mechanisms}.

\paragraph{Performativity.}
Our work relates to the emerging field of 
\emph{performative prediction} \citep{perdomo2020performative}, 
which studies learning in settings where model deployments 
affect the underlying data distribution,
with emphasis on (re)training dynamics and equilibrium.
A central effort in this area is to identify general global properties
(e.g., appropriate notions of smoothness) that guarantee convergence.
Our work complements this effort by proposing \emph{structure},
in the form of user self-selection \cite{zhang2021classification,horowitz2024classification},
which provides an efficient low-dimensional representation of the distribution map.
This allows us to work in the
notoriously challenging \emph{stateful} performative setting \citep{brown2022performative},
and to characterize stronger notions of stability.
\squeeze 

\paragraph{Fairness.}
The notion of equilibrium in evolutionary prediction games relates to
the fairness criterion of \emph{overall accuracy equality} %
\citep{verma2018fairness}.
The main distinction is that in prediction games,
the question is not only \emph{whether} groups are treated fairly,
but also \emph{which} groups are even considered,
since fairness can be measured only for groups that are observed.
Complementing other voiced concerns on long-term fairness outcomes
\citep{hashimoto2018fairness, liu2018delayed,d2020fairness,raab2021unintended},
our work adds a novel  
counterfactual perspective:
if we observe fair outcomes at present,
could this be because some groups were historically driven out of the game?

See \Cref{appendix:related} for an extended discussion and additional related literature.

\begin{figure*}
    \centering
    \includegraphics[width=\textwidth]{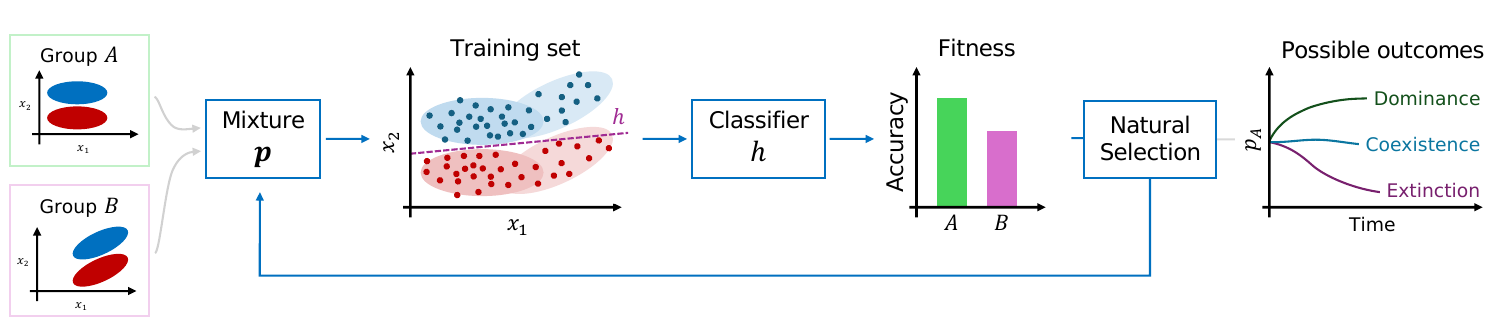}
    \caption{
    Natural selection in a two-group setting.
    Population is a mixture $\vp$ of groups;
    Classifier $h$ is learned using data from the mixture distribution $D_\vp$;
    Evolutionary fitness is associated with prediction accuracy;
    Differences in fitness drive change in mixture coefficients;
    Possible long-term tendencies are dominance (only $A$ survives), extinction (only $B$ survives), or coexistence (both survive).
    }
\label{fig:interaction_model_schematic_diagram}
\end{figure*}

\section{Setting}
\label{sec:setting}

The core of our setting is based on a standard supervised learning setup
in which examples describe user data.
Denote features by $x\in\Features$ and labels $y\in\Labels$.
For a given distribution $D$
over pairs $(x,y)$,
and given a training set $S = \Set{(x_i,y_i)}_{i=1}^n\sim D^n$,
the goal in learning is to use $S$ to find a predictor $h : \Features \to \Labels$
from a class $\Hypotheses$
which minimizes the expected error
under some loss function $\ell:\Labels\times\Labels\to\Reals$.
For concreteness, we focus on classification tasks with the 0-1 loss,
$\ell(\yhat,y)=\Indicator{\yhat \neq y}$,
whose minimization
is equivalent to 
maximizing expected accuracy:
\squeeze
\begin{equation}
\label{eq:expected_accuracy}
\acc_D(h)=
\prob{(x,y)\sim D}{h(x)=y}
\end{equation}
Denote the learning algorithm by $\Learner(S)$,
and by $h\sim\Learner(D)$ the process mapping distributions to learned classifiers,
with randomness due to sampling.
Since solving Eq.~(\ref{eq:expected_accuracy}) is both computationally and statistically hard, practical learning algorithms often resolve to optimizing an empirical surrogate on the sampled data.
One question we will ask is how such compromises affect long-term outcomes.
\squeeze

\paragraph{Population structure.}
As
data points $(x,y)$ are generated by users,
the data distribution $D$
represents a population.
We assume the population is a mixture of $K$ groups.
Each group $k\in[K]$ is associated with a group-specific distribution $D_k$ over $(x,y)$ pairs, which remains fixed, and with a relative size $p_k$, which evolves over time.
The overall data distribution is a mixture denoted by
$D=D_\vp=\sum_{k} p_k D_k$,
where $p_k$ is the current relative proportion of group $k$ (such that $\sum_k p_k = 1$ and $p_k \ge 0 \,\, \forall k$).
Given some $h$,
the marginal expected accuracy on group $k$ is 
$\acc_k(h)=\acc_{D_k}(h)$.
The \emph{population state} vector $\vp=(p_1,\dots,p_K)$
will be our main object of interest.
As $\vp$ encodes group proportions,
the space of population states is the $(K-1)$-dimensional simplex $\DistOverGroups$,
and since $\vp$ determines $D_\vp$, for brevity we denote $h \sim \Learner(\vp)=\Learner(D_\vp)$
and $\acc_\vp(h) = \acc_{D_\vp}(h)$.

\paragraph{Dynamics.}
When a classifier is deployed, users make decisions according to the quality of the predictions they receive. Collectively, individual responses reshape the population composition 
$\vp$, resulting in a \emph{subpopulation shift} \citep{yang2023change}.
The updated population then serves as the training distribution for the next classifier---which forms a feedback loop between $\vp$ and $h$.
For training, we consider a simple procedure of repeated training
on the current distribution.
In terms of user behavior,
our framework accommodates a wide family of dynamics,
including reproduction, network effects, %
imitation, social influence, rational decision-making, and competition,
as alluded to in \Cref{sec:intro}.
In Appendix~\ref{appx:microfoundations} we define these formally and
provide additional examples.
In \Cref{sec:prediction_games} we make our assumptions explicit.
\squeeze

\paragraph{Outcomes.}
Denote the initial population state by $\vp^0$, and the initial classifier by $h^0\sim\Learner(\vp^0)$. 
We consider dynamics which evolve towards fixed points, and denote the corresponding point by $\vp^*$.
Denote the classifier by $h^*\sim \Learner(\vp^*)$. We characterize fixed points in terms of three key properties:
\begin{itemizecompact}
    
\item 
\textbf{Population composition:}
For the population, dynamics may act as a force
diminishing
certain groups (i.e., driving towards a state satisfying $p^*_k = 0$ for some $k\in[K]$). In such cases, we say that the population is \emph{dominated} by the remaining groups. When the population composition at rest has at least two groups (i.e., $p^*_k>0$ for multiple values of $k$), we say that these groups \emph{coexist}.

\item
\textbf{Accuracy:}
For the classifier,  dynamics may result in overall performance 
improving over time,
(i.e.,~$\acc(h^*)>\acc(h^0)$),
deteriorating, or remaining unchanged.

\item
\textbf{Stability:}
For the joint system, some fixed points $\vp^*$ are \emph{stable} and attract all neighboring states, while other fixed points are unstable and may cause dynamics to diverge under small perturbations.

\end{itemizecompact}

An illustration of the setting is provided in \Cref{fig:interaction_model_schematic_diagram}.

\section{Evolutionary Prediction Games}
\label{sec:prediction_games}

\begin{figure*}
    \centering
    \includegraphics[width=\textwidth]{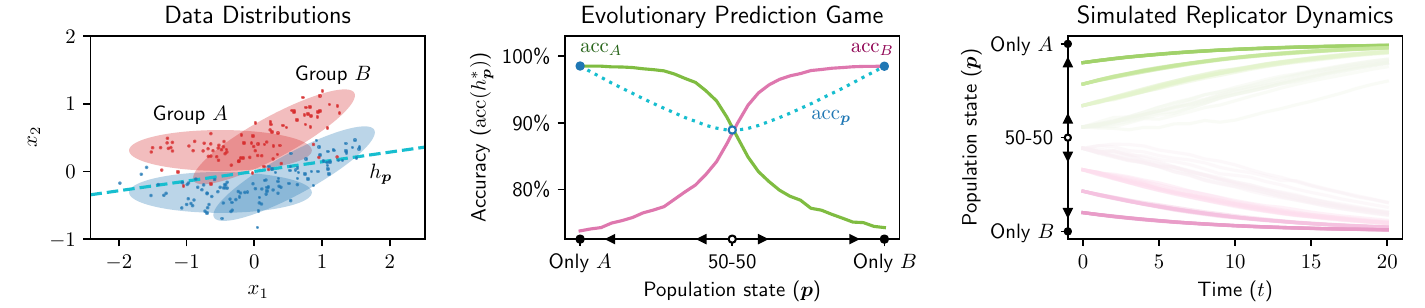}
    \vspace{-1em}
    \caption{
    Evolutionary dynamics for two groups, induced by an oracle linear classifier (\Cref{sec:retraining_optimal_classifiers}).
    \textbf{(Left)} Data distributions in feature space. 
    Dashed line demonstrates the optimal linear classifier $h_\vp$ for the uniform mixture $\vp=(0.5,0.5)$. 
    \textbf{(Center)} Evolutionary prediction game $F_k(\vp)=\acc_k(h_\vp)$. 
    The game has two stable equilibria with single-group dominance, and one unstable coexistence equilibrium. %
    Overall population accuracy $\acc_\vp(h_\vp)$ is convex, and maximized at the boundries.
    \textbf{(Right)} Replicator dynamics induced by the game, for various initial states $\vp^0$. Populations evolve towards fixed points with single-group dominance.
    }
\label{fig:optimal_classifier_retraining}
\end{figure*}

To analyze the feedback loop, we leverage \emph{evolutionary game theory}, which is a general framework for the analysis of natural selection processes~\citep{smith1973logic}.
The framework models evolutionary interactions between individuals from $K$ groups in a large mixed population:
Given a population state $\vp\in\DistOverGroups$, each group $k \in [K]$ 
is associated with a scalar \emph{fitness function}, $F_k(\vp)$,
which quantifies the average evolutionary fitness of individuals in that group.
Together, the fitness functions $F=\left(F_1,\ldots,F_K\right)$ define a \emph{population game} $F:\DistOverGroups\to\Reals^K$, which governs the dynamics of $\vp$. 

Formally, population games are symmetric normal-form games for a continuum of agents.
Evolutionary game theory uses population games to map
 myopic, often local agent decisions
to group-level dynamics governing the evolution of group proportions $\vp$ over time.
Modeling the game at the group level serves as useful abstraction for studying the general tendencies and possible long-term outcomes of local interactions in aggregate.
Note that the notion of ``group'' is flexible and can support memberships that are either inherent (e.g., by demographics) or chosen (e.g., service provider).

\paragraph{Accuracy as evolutionary fitness.}
We formalize the interplay between learning and population dynamics through a novel type of population game---%
an \emph{evolutionary prediction game}.
The game associates
the evolutionary fitness of each group with its expected marginal accuracy $\acc_k$ under a classifier $h$ trained on data sampled from the mixture $D_\vp$:
\begin{definition}[Evolutionary Prediction Game] \label{def:prediction_game}
Let 
$\Learner$ be a learning algorithm%
, and let 
$D_1,\dots,D_K\in\DistOver{\FeaturesLabels}$ be group distributions.
Denote by $\acc_k(h)$ the marginal accuracy of group $k\in[K]$ under classifier $h \sim \Learner(\vp)$.
The \emph{evolutionary fitness} of group $k$ under population state $\vp$ is:
\begin{equation}
    \label{eq:population_game_fitness}
    F_k(\vp)
    = 
    \expect
    {
    h\sim \Learner(\vp)
    }
    {\acc_{k}(h)}
\end{equation}
Together,
the tuple $F(\vp) = (F_1(\vp),\dots,F_K(\vp))$ defines an \emph{evolutionary prediction game}.
\end{definition}

In the notation $F_k(\vp)$, we assume for brevity that the learning algorithm $\Learner$ and group distributions $D_k$ are clear from context, and we state them explicitly otherwise.

Evolutionary prediction games capture an implicit interaction between groups:
while each group's fitness $F_k(\vp)$ is given by
the accuracy of its own members,
the classifier $h_\vp$ is trained using data from the entire population.
Thus, the evolution of different groups 
becomes coupled through the learning process.
For example,  \Cref{fig:optimal_classifier_retraining} (Left) illustrates a binary classification setting over two groups $\Set{A,B}$. \Cref{fig:optimal_classifier_retraining} (Center) shows the corresponding evolutionary prediction game $F(\vp)$ induced by an oracle linear classifier (defined in \Cref{sec:retraining_optimal_classifiers}). At $\vp=(0.7,0.3)$, we have $F_A(\vp) \approx 0.97$ and $F_B(\vp)\approx 0.78$, and thus the fitness of group $A$ is higher when the classifier is trained on data sampled from $D_{\vp}$.
\squeeze

\paragraph{Equilibrium.} %
Nash equilibrium in general population games is defined as follows:
\begin{definition}[Nash equilibrium of a population game; {e.g. \citep{sandholm2010population}}]
\label{def:nash_equilibrium}
Let $F(\vp)$ be a population game, and denote $\support(\vp)=\Set{k\mid p_k>0}$. A state $\vp^*\in\DistOverGroups$ is a Nash equilibrium of $F(\vp)$ if it satisfies:
\begin{equation}
\label{eq:nash_equilibrium}
\support(\vp^*)\subseteq\argmax\nolimits_{k\in[K]} F_k(\vp^*)  
\end{equation}
\end{definition}
This also applies to evolutionary prediction games;
e.g., the game in \Cref{fig:optimal_classifier_retraining} (Center) has three Nash equilibria at $\Set{(1,0),(0.5,0.5),(0,1)}$, and
all population games have at least one equilibrium 
\citep{sandholm2010population}.

\paragraph{Induced population dynamics.}
Formally, we encode dynamics as $\dot\vp=V_F(\vp)$, where $\dot\vp$ is the time derivative of $\vp$, and $V_F:\DistOverGroups\to T \DistOverGroups$ is a tangent vector field induced by the game $F(\vp)$.
Our results apply under mild structural assumptions on $V_F$, namely:
(i) $V_F$ is continuous; (ii) the direction of flow is aligned with fitness, formally $V_F(\vp)\cdot F(\vp)>0$ whenever $V_F\neq 0$, a property known as \emph{positive correlation}; (iii) Fixed points of the dynamical system coincide with equilibria of $F(\vp)$, formally through either \emph{Nash stationarity} or \emph{imitative dynamics}.
All are satisfied by the user-level dynamics highlighted in \Cref{sec:intro} and \Cref{appx:microfoundations},
where they emerge as group-level properties
in the large-population limit.
Moreover, many of them converge towards the canonical
\emph{replicator equation},
$\dot p_k=p_k\left(F_k(\vp)-\bar F(\vp)\right)$,
where $\bar F(\vp)=\sum_{k'}p_{k'}F_{k'}$ is the average fitness across the population \citep{taylor1978evolutionary}.
We also note 
that time-scales of convergence are determined (and can be adjusted) by the scale of $V_F$.
See \Cref{subsec:appendix_dynamics} for formal definitions.

\paragraph{Fairness.} In the context of fairness, a classifier satisfies \emph{overall accuracy equality} if all groups receive equal prediction accuracy \citep{verma2018fairness}. 
We show that this criterion is satisfied in expectation by classifiers trained on equilibria mixtures of evolutionary prediction games (proof in \Cref{sec:equilibria_are_fair}):
\begin{proposition}
\label{prop:equilibrium_fairness}
Let $F(\vp)$ be an evolutionary prediction game induced by a learning algorithm $\Learner$, and let $\vp^*$ be a Nash equilibrium. 
Then $h\sim \Learner(D_{\vp^*})$ satisfies overall accuracy equality in expectation.
\end{proposition}
As outcomes of the dynamics correspond to equilibria of a game, this lends to a natural form of emergent fairness. 
However, note that this criterion only applies to groups that appear in the data: 
in \Cref{subsec:fairness_empirical}, we 
further
show 
that a system may appear as fair 
due to other groups being driven out.

\paragraph{Heterogeneous fitness.} 
Finally, in some settings, incentives or outside alternatives may vary across groups. This may be captured using the notion of \emph{retention functions} \citep{hashimoto2018fairness},
which translate marginal prediction accuracy to the tendency to remain engaged. In such cases, the
transformed fitness function
becomes $F_k(\vp)=\Retention_k\left(\expect{h_\vp}{\acc_k(h)}\right)$, where $\Retention_k:[0,1]\to\Reals$ is a group-dependent, continuous strictly-increasing function.
A simple example is $\Retention_k(x)= x - b_k$ (for $b_k>0$)
which allows groups to differ in the minimum accuracy their users require to keep using the system.
All of our results apply to any set of strictly monotone $\Retention_k$ for $K=2$, and to positive affine $\Retention_k$ with per-group offsets $b_k$ for $K>2$.
For clarity we state our main results for $\Retention_k(x)=x$,
but also discuss them for broader $\Retention$.

\section{Competitive Exclusion Under Oracle Classifiers}

\label{sec:retraining_optimal_classifiers}

To understand how populations and classifiers evolve jointly, we start by characterizing natural selection in an `ideal' setting, in which a classifier is trained repeatedly using unlimited resources:

\begin{definition}[Oracle classifier]
Let $\Hypotheses$ be a hypothesis class, and let $D$ be a data distribution. An \emph{oracle classifier} with respect to $\Hypotheses$ is a 
minimizer of the 0-1 loss in expectation over $D$:
\squeeze
\begin{equation}
\label{eq:optimal_classifier}
h^\opt \in \argmin\nolimits_{h\in\Hypotheses} 
\expect{(x,y)\sim D} {\Indicator{h(x)\neq y}}
\end{equation}
\end{definition}
Oracle classifiers represent ERM classifiers in the population limit. This regime abstracts away the complexities due to estimation and approximation errors of practical learning algorithms (i.e., algorithms that learn from finite data in reasonable time).
Nonetheless, 
analysis of the dynamics remains challenging due to the $\mathrm{argmin}$ operator.
We focus on settings in which $h^\opt$ exists, and denote by $\Learner^\opt(D)$ the (theoretical) learning algorithm which returns the oracle classifier with respect to $D$.%
\footnote{
The Bayes-optimal classifier $h^{\mathrm{Bayes}}(x)=\argmax_y p(y\,|\,x)$ is a special case of this definition
for
$\Hypotheses=\Labels^\Features$.
\squeeze} 
We assume that fitness functions are continuous,
and tie-breaking is consistent.
To simplify presentation, we first assume that each group has a distinct oracle classifier, and then address the general case.
Our central result characterizes natural selection induced by oracle classifier retraining:

\begin{theorem}
\label{thm:optimal_learner}
Let $\Hypotheses$ be a hypothesis class, and denote by $\Learner^\opt$ the oracle learning algorithm with respect to $\Hypotheses$.
Assume that
at each population state $\vp$
the oracle classifier $h_\vp\sim\Learner^\opt(D_\vp)$ is deployed. 
Then it holds that:
\begin{enumeratecompact}
    \item \textbf{Accuracy:} 
    Overall accuracy increases over time,
    $
    \tfrac{\mathrm d}{\mathrm dt}
    \acc_{\vp}(h_{\vp}) \ge 0
    $.
    \item \textbf{Stability:} A stable equilibrium always exists,
    and there can be multiple such equilibria.
    \item \textbf{Competitive exclusion:} 
    For all stable equilibria,
    $\Size{\support(\vp^*)}=1$.%
    \item \textbf{Coexistence:}
    Equilibria with $\Size{\support(\vp^*)}\ge2$ may exist,
    but are unstable.

\end{enumeratecompact}
\end{theorem}

\paragraph{Proof sketch.}
The proof leverages structural properties that emerge despite the complexity of the learning problem.
Our main technical tool
is the framework of
\emph{potential games}, which are population games that can be expressed as a gradient of a scalar potential function defined over the simplex.
We start by showing that the expected accuracy of any fixed classifier is linear in $\vp$, then the core of the proof is a convexity argument: the optimality of $\Learner^\opt$ implies that $\acc_\vp(h_\vp)$ is convex as a pointwise maximum of linear functions, with gradients given by marginal accuracies. From this we identify $f(\vp)=\acc_\vp(h_\vp)$ as a potential function, and leverage the known correspondence between equilibria of population games and local extrema of potential functions.
Finally, distinct oracle classifiers imply strict convexity of $f(\vp)$, and the stability of single-population states follows from the fact that convex functions over the simplex are locally-maximized at the vertices, and the accuracy condition follows from the fact that $f(\vp)$ is a Lyapunov function.
Full proof appears in \Cref{sec:retraining_proofs}.

\paragraph{Extensions.} To simplify presentation, the statement of \Cref{thm:optimal_learner} assumes that the oracle classifiers for each group are distinct. When this assumption doesn't hold (e.g., when $D_k=D_{k'}$ for some $k\neq k'$), a generalized version of the theorem applies: rather than single-group stable equilibria, dynamics converge towards convex combinations of groups that share an identical oracle classifier with identical marginal accuracy (See \Cref{subsec:appendix_equivalent_groups}). 
For heterogeneous fitness where retention functions $\Retention_k$ may differ across groups, we show that the retention functions we support induce potential games with transformed potential functions. Survival is determined according to the maximizers of the transformed potential function, and the quantity that increases over time becomes the average fitness $\sum_{k} p_k \Retention_k\left(\acc_k(h_\vp)\right)$ (See \Cref{subsec:appendix_retention_functions}).
Finally, we also note that a similar characterization also holds for classifiers trained once at the outset (See \Cref{sec:training_once_proofs}).

\paragraph{Interpretation.}
\Cref{thm:optimal_learner} suggests that for repeatedly trained oracle classifiers, natural selection will amplify the dominance of groups having the highest fitness at the outset. From an evolutionary perspective, this reflects the \emph{competitive exclusion principle} \citep{hardin1960competitive},
which states that natural selection tends towards exclusive survival of the fittest when multiple species compete over the same resources. 
Stability implies that once a group dominates,
other groups will remain excluded. 
While this is the general trend,
coexistence is nonetheless \emph{enabled} by retraining,
in the form of possible equilibria in which multiple groups survive.
That is, 
there may exist population states $\vp^*$ such that $p^*_k>0$ for several groups $k$, and for which this remains to hold when the classifier is retrained.
The crux, unfortunately, is that such points %
are \emph{evolutionarily unstable}, in the sense that small perturbations around $\vp^*$ will push dynamics away.
This implies that optimal retraining---unconstrained by limited data and  compute---induces a strong form of competition between groups.
But the existence of mixed equilibria suggests that other algorithms might
be able to encourage symbiosis,
which we explore next.
\squeeze

\section{Avenues for Coexistence}
\label{sec:coexistence}

\begin{figure*}
    \centering
    \includegraphics[width=\textwidth]{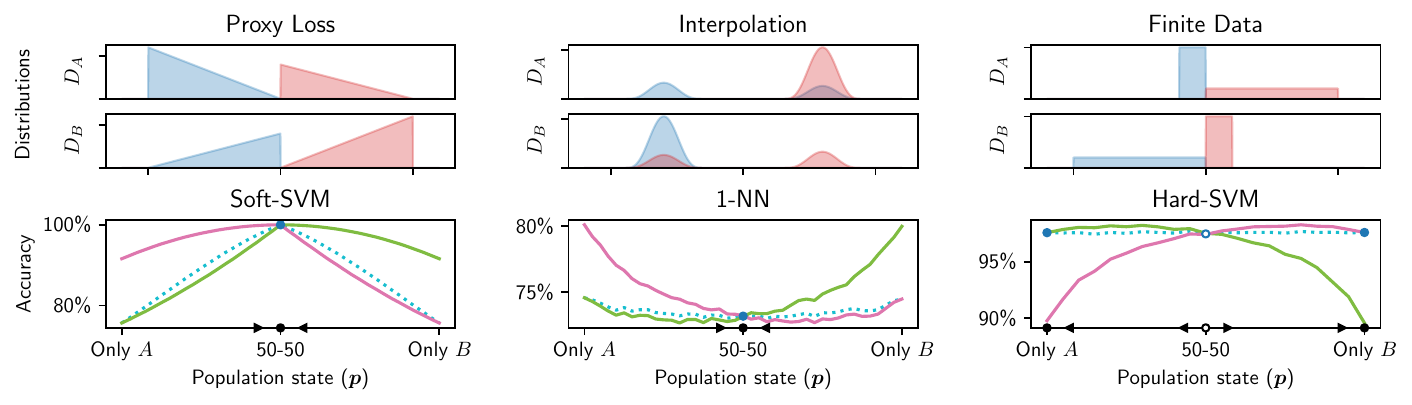}
    \vspace{-1em}
    \caption{
    Graphical illustration of theoretical coexistence results. Top row presents data distributions, bottom row shows the corresponding prediction games.
    \textbf{(Left)} Stable and mutualistic coexistence induced by the use of a proxy loss 
    (\Cref{thm:soft_svm_coexistence}).
    \textbf{(Center)} Stable coexistence induced by 
    interpolation 
    (\Cref{subsec:knn_coexistence_proof}).
    \textbf{(Right)} Mutualistic coexistence induced by 
    a finite training set
    (\Cref{subsec:finite_data_coexistence_proof}).
    }
\label{fig:coexistence_mechanisms}
\end{figure*}

Intuitively, oracle classifiers drive dynamics away from mixture equilibria because
the learning objective is fully aligned with, and so reinforces, 
self-selection.
This motivates the natural question:
Can repeated training drive dynamics also towards coexistence, and when?
Here we illustrate 
two ways for symbiosis to arise:
as a result of inherent limitations of learning in practice
(i.e., limited data or compute),
and by applying an ad-hoc, dynamics-aware learning algorithm
designed to stabilize outcomes.
\squeeze

\paragraph{A context-dependent view on coexistence.}
It is important to emphasize that coexistence in our setting is neither innately `good' or `bad',
but rather, its desirability depends on context.
Hence, the question of whether coexistence \emph{should} be encouraged depends on the task at hand,
the nature of user choices,
and the goals of the system designer.
For instance, recommendation platforms may aim to preserve both content and audience diversity, making coexistence desirable. 
Conversely,
in consumer markets having multiple competing standards
(e.g., video encoding format, compression schemes, payment platforms),
there are benefits to converging on the single best alternative.
This context-dependent stance aligns well with the notion of evolution,
which in itself is neutral.

\subsection{Learning in Practice}
\label{subsec:learning_in_practice}
While the goal in learning is to minimize the expected 0-1 loss,
in practice learning often resorts to optimizing proxy objectives over sampled data
via empirical risk minimization.
Here we show that, perhaps surprisingly,
basic aspects of this approach---%
namely the use of surrogate losses,
access to finite data,
and interpolation methods---can create conditions that facilitate \emph{mutualism}, in which all groups benefit from coexistence.
For each aspect,  we demonstrate
through a carefully crafted example
how it can act as a mechanism that enables stable coexistence.
Here we focus our discussion on surrogate loss as a representative construction;
for finite data and memorization,
we illustrate the underlying distribution and induced game in 
\Cref{fig:coexistence_mechanisms},
and defer the formal analysis to \Cref{subsec:soft_svm_coexistence_proof}.

\paragraph{Coexistence by surrogate loss.} 
A common approach to empirical risk minimization is to replace the 0-1 loss
with a convex surrogate loss, such as the \emph{hinge loss} used in the classic Soft-SVM algorithm \citep{cortes1995support}.
This enables optimization, but introduces bias:
whereas the 0-1 loss applies uniform penalties,
surrogates penalize misclassified points relative to their distance
from the decision boundary of $h$.
Our next result uses this artifact to construct a game
having two groups that are complementary in their incurred bias,
and hence balance each other to induce mutualistic coexistence.
\squeeze

\begin{theorem}
\label{thm:soft_svm_coexistence}
There exists an
evolutionary prediction game in which
learning with
the hinge loss and $\ell_2$ regularization
induces a mixed equilibrium that is both stable and fitness-maximizing.
\end{theorem}

The construction is outlined in \Cref{fig:coexistence_mechanisms} (Left), and proof is provided in \Cref{subsec:soft_svm_coexistence_proof}.
The proof leverages the bias of hinge loss against minority classes:
For each group alone,
minimizing the hinge loss biases the learned classifier against the minority class, which is suboptimal.
But since each group has a different minority class,
these biases cancel out at the mixed equilibrium.
The surprising property is that this state is also evolutionarily stable:
when one group grows slightly
(i.e., $p_1=0.5+\epsilon$), it loses more from the fact that the other group shrinks (i.e., $p_2=0.5-\epsilon$) than it gains from its own growth.
In this sense, each group needs the other to maintain its existence,
giving rise to 
stable and mutualistic 
symbiosis.
\Cref{subsec:svm_regularization_experiment} expands on this idea
to show that varying the amount of regularization leads to \emph{bifurcations},
in which the game transitions between 3, 1, and 5 equilibria.%
\squeeze

\paragraph{Interpretation.}
In the example game above,
as in those we construct for finite data (\Cref{subsec:finite_data_coexistence_proof}; \Cref{fig:coexistence_mechanisms} (Right)) and interpolation (\Cref{subsec:knn_coexistence_proof}; \Cref{fig:coexistence_mechanisms} (Center)),
mutualism is enabled by how user self-selection
compensates for the algorithm's imperfections.
When the resulting bias is complementary across groups,
this can give rise to (implicit) cooperation.
One interpretation is that the existence of each group acts as regularization on others:
from the perspective of group $k$,
the objective can be written as
$\argmax_h \acc_k(h) + \lambda R(h)$ where
$R(h) = \sum_{j \neq k} \acc_k'(h)$ is a data-dependent regularizer with coefficient $\lambda=\frac{1}{p_k}\sum_{j \neq k} p_j$.
Nonetheless, our results should not be taken to imply that proxy losses or finite data are a good \emph{means} to achieve coexistence;
rather, they highlight how mutualism can organically materialize when learning affects different groups differently.
We explore this further empirically in 
\Cref{subsec:mnist_experiment}.

\subsection{Stabilizing Coexistence Equilibria}
\label{subsec:stabilization}

Thus far we have considered evolutionary prediction games in which learning pursues the conventional objective of maximizing accuracy at each timestep on the current distribution.
But if the learner is aware of population dynamics and of how the learned classifier may influence outcomes,
then it makes sense to consider learning algorithms that take this into account.
The question of course is how.
\squeeze

Here we focus on a particular aspect of this general question that
arises from our last result in Sec.~\ref{sec:retraining_optimal_classifiers},
namely:
given a system with a desirable but unstable mixed equilibrium $\vp^*$,
how do we stabilize it?
We propose a conceptually simple solution that works by
reweighing examples in the standard
accuracy objective in a way that accounts for dynamics.
The main idea is to invert the natural tendency of dynamics to push away from the unstable equilibrium, and instead pull towards it, by training `as if' the distribution was at a different state $\vp'$ than its actual current state $\vp$:

\begin{proposition}
\label{prop:stabilizing_coexistence}
Let $\Learner^\opt(\vp)$ be an oracle algorithm with equilibrium $\vp^*$.
If $\vp^*$ has full support (and thus is unstable),
then it becomes stable under $\Learner'(\vp)=\Learner^\opt(2\vp^*-\vp)$.
\end{proposition}
Proof in \Cref{subsec:steering_coexistence_proof}.
The idea is that $\Learner'$ trains `as if' the state is $2\vp^*-\vp$,
which can be achieved by weighing examples from each group $k$
by $w_k=\frac{p_k}{2p^*_k-p_k}$.
This inverts the natural tendency of dynamics to push away from the unstable equilibrium, and instead pull towards it. 
The statement regards optimal classifiers and requires knowledge of $\vp^*$,
but in practice we can set $w$ according to an estimated
$\tilde{\vp}$.
We further explore this approach in \Cref{subsec:cifar_experiment}.

\section{Experiments}
\label{sec:empirical}

To demonstrate how some of the principles we have discussed
so far apply in more realistic settings, we now turn to explore evolutionary prediction games
empirically using real data and simulated dynamics.
We include three experiments.
The first experiment complements our results on 
Sec.~\ref{subsec:stabilization} by demonstrating how stabilization can improve both overall and per-group accuracies via coexistence.
The second complements 
Sec.~\ref{subsec:learning_in_practice}, showing how mutualism can arise organically due to group balancing,
here from complementarities in label noise.
The third experiment explores fairness under population dynamics (Prop.~\ref{prop:equilibrium_fairness})
to illustrate the limitations of static fairness measures.

\subsection{Mutualistic Coexistence From Data Augmentation}
\label{subsec:cifar_experiment}
Data augmentation is a method for enriching the training set with 
class-preserving transformations of inputs
(e.g., camera angle, source of lighting) \citep[e.g.,][]{shorten2019survey}.
Augmentations are typically considered as artificial constructs,
but there are also cases where they emerge naturally
(e.g., medical images from different hospitals or imaging devices).
Here we show that when different sources correspond to different groups, data pooling can serve to foster beneficial coexistence
in which all groups gain from the existence of others.
This suggests that `natural' augmentations can be effective also in the long run.
\squeeze

\paragraph{Method.} We use the CIFAR-10 image recognition dataset \citep{krizhevsky2009learning}, which consists of 60,000 32x32 color images in 10 classes, with 6,000 images per class. 
As horizontal image flips are considered class-preserving for this dataset, group $A$ consists of images sampled directly from CIFAR-10, and group $B$ consists of images under horizontal flip. We use a ResNet-9 network for prediction \citep{he2016deep}, and train it using the \texttt{ffcv} framework with default optimization parameters \citep{leclerc2023ffcv}. For each $\vp$, we measure the prediction model's accuracy on the original images from the CIFAR test set (representing $\acc_A(h_\vp)$), and on their flipped counterparts (representing $\acc_B(h_\vp)$).
For stabilization, we use the method described in \Cref{subsec:stabilization} with $\vp^*=(0.5,0.5)$. We simulate the evolutionary process using discrete replicator dynamics, and use linear interpolation to determine intermediate fitness values. 

\paragraph{Results.}
\Cref{fig:empirical} (Left) presents the estimated evolutionary prediction game, averaged across $\ParamEmpiricalCifarAugmentationNRepetitions{}$ repetitions.
The game has three equilibria: Two stable equilibria with single-group dominance
($\ParamEmpiricalCifarAugmentationSingleGroupAccuracyConfidence{}$ accuracy),
and an unstable coexistence equilibrium ($\ParamEmpiricalCifarAugmentationCoexistenceAccuracyConfidence{}$ accuracy).
Although fitness is typically higher for the larger group,
the mixed equilibrium attains the highest overall population accuracy,
suggesting that each groups benefits from the existence of the other.
Reaching this point however requires stabilization;
\Cref{fig:empirical} (Center) shows how 
our algorithm from \Cref{subsec:stabilization}
is able to attain stable beneficial coexistence,
despite using an empirical proxy objective.
\Cref{app:additional_experiments} includes additional results on
sampling noise,
time to convergence,
and sensitivity analysis for stabilization.
\squeeze

\begin{figure*}
    \centering
    \includegraphics[width=\textwidth]{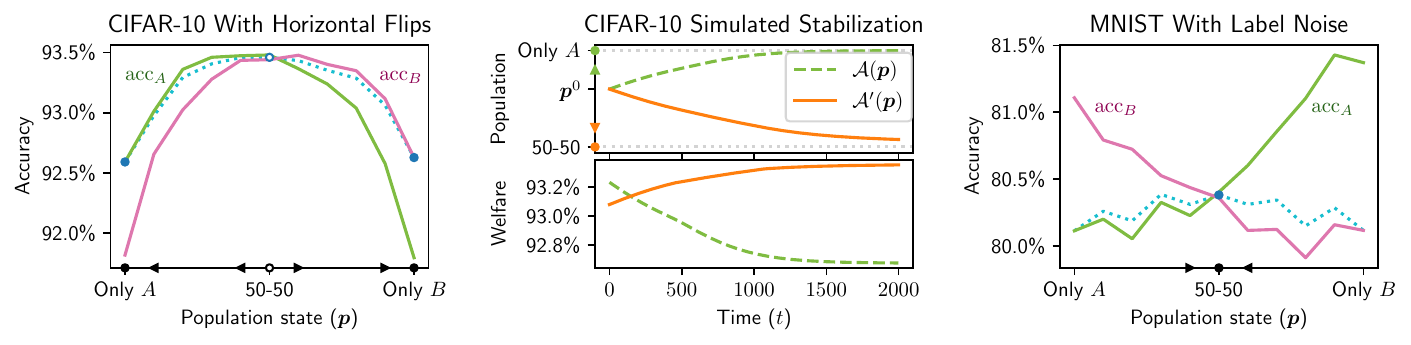}
    \vspace{-1em}
    \caption{
    Empirical evaluation in two-group settings.
    \textbf{(Left)}
    Game induced by
    CIFAR-10 with Resnet-9 and groups representing horizontal flips. The game has unstable mutualism (\Cref{subsec:cifar_experiment}).
    \textbf{(Center)}
    Simulated stabilization of replicator dynamics on the CIFAR-10 game (\Cref{subsec:stabilization}).
    \textbf{(Right)} 
    Game induced by
    MNIST with label noise, showing stable mutualism (\Cref{subsec:mnist_experiment}).
    }
\label{fig:empirical}
\end{figure*}

\subsection{Stable Coexistence From Overparameterization}
\label{subsec:mnist_experiment}

Modern neural networks are frequently overparameterized, and are known to achieve strong test-time performance despite memorizing the training data \citep[e.g.,][]{zhang2021understanding}. Recent research attributes this phenomenon to the implicit biases introduced by the training process \citep[e.g.,][]{vardi2023implicit}. 
Building on these ideas, here we demonstrate that stable coexistence can emerge when an overparameterized neural network is trained on
data in which groups have complementing label noise.
Our construction considers a multi-class settings where each group includes a majority of \emph{noisy} examples from some class,
and a minority of clean examples from other classes.
This aims to captures settings in which user groups
(e.g., by country) differ in the marginal distribution $p(y)$ over classes (e.g.,  spoken language),
and where the dominant class is more likely to be misclassified
(e.g., due to many accents).
\squeeze

\paragraph{Method.} 
We use the MNIST dataset, set $K=2$, and split the classes unevenly between two groups:
Group $A$ is biased towards even digits $\Set{0,2,4,6,8}$, whereas group $B$ is biased towards odd digits, both with a 4:1 imbalance. For each group, we introduce label noise to the majority classes, mapping the true label of each digit $d$ to the next digit with same parity with probability $0.2$ (i.e. for group $A$, label noising stochastically maps $0\mapsto2$, $2\mapsto4$, etc.). We train a convolutional neural network for 200 epochs using stochastic gradient decent with momentum. The training and test sets are split independently, and for each $\vp$ we measure the prediction model's accuracy on both splits of the test sets (representing $\acc_A(h_\vp)$ and $\acc_B(h_\vp)$). Under the given parameters, the label noising process leads to an accuracy upper bound of $84\%$ assuming perfect recognition.

\paragraph{Results.}
\Cref{fig:empirical} (Right) shows the estimated evolutionary prediction game, averaged across $\ParamEmpiricalMnistLabelNoiseNRepetitions{}$ repetitions.
Note how label noise has 'flipped' the game in that accuracy is higher for the minority group
(i.e., $\acc_B>\acc_A$ when $p_B<p_A$, and vice versa).
This indicates that groups naturally balance each other in this setting.
As a result, the game has a stable coexistence equilibrium,
with $\ParamEmpiricalMnistLabelNoiseCoexistenceAccuracyConfidence{}$ test accuracy
(vs. the theoretical upper bound of $84\%$ due to label noise).
Training accuracy is $\ParamEmpiricalMnistLabelNoiseTrainAccuracyConfidence{}$,
suggesting that interpolation is indeed the likely phenomenon at play.
See \Cref{subsec:knn_coexistence_proof} for further discussion on how interpolation facilitates coexistence as an underlying mechanism.
\squeeze

\subsection{Population Dynamics and Fairness}
\label{subsec:fairness_empirical}

\begin{figure*}
    \centering
    \includegraphics[width=\textwidth]{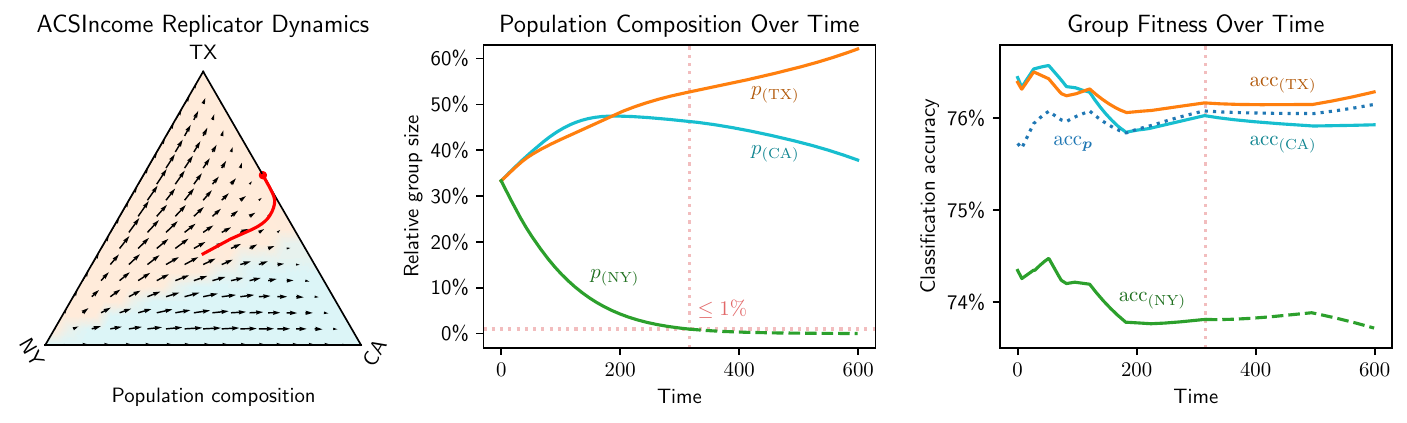}
    \vspace{-1em}
    \caption{
    Evolutionary dynamics in a three-group setting, induced by a linear SVM trained on the ACSIncome dataset (\Cref{subsec:fairness_empirical}).
    \textbf{(Left)}
    Replicator dynamics on the three-group simplex. Red line indicates a trajectory starting from the uniform mixture, and background colors indicate basins of attraction.
    \textbf{(Center)}
    Population composition over time for the same trajectory. Proportion of NY users becomes negligible at $t \approx \ParamFolktablesSvmDisparityTime$.
    \textbf{(Right)} 
    Fitness over time. Effective group disparity is significant at the outset, but shrinks by an order of magnitude after NY users are driven out.
    }
\label{fig:folktables_fairness}
\end{figure*}

Lastly, we examine how population dynamics interact with notions of fairness.
Our goal here is to show that a system can appear to be fair at a given point in time---but only because some groups have been driven out of the game, and fairness is measured on the remaining observable groups.
Our message here is that fairness should consider not only what has happened, but also what \emph{could} have happened, particularly with respect to groups that perhaps had previously been excluded.

\paragraph{Method.} We use a subset of the Folktables ACSIncome dataset \citep{ding2021retiring}, a tabular income-level classification task based on US census data. We associate groups with data from three states in the US, with the idea that firms operate in different states but train a model on data aggregated from all sources. 
Specifically, we set $K=3$, and use 2018 census data from the states of California (\ParamFolktablesStateCaliforniaNumDatapoints{} datapoints), 
New York (\ParamFolktablesStateNewYorkNumDatapoints{}), 
and Texas (\ParamFolktablesStateTexasNumDatapoints{}). 
We use linear SVM as the learning algorithm, and randomly hold out \ParamFolktablesTestSetSize{} points from each state for evaluating accuracy. 
To compute fitness at a given population state $\vp$,
we sample a dataset of size $n=\ParamFolktablesTrainingSetSize$ from the mixture, train a classifier, and measure group accuracies using the held-out sets. 
Results are averaged over $\ParamFolktablesGridNumRepetitions$ repetitions to obtain mean fitness values, and we simulate replicator dynamics starting from the uniform population state.
\squeeze

\paragraph{Results.}
\Cref{fig:folktables_fairness}
shows how
evolution transitions between two dynamical phases: an initial phase 
(up to time step ${\sim}200$) where one group (NY) diminishes while the other two (CA,TX) remain balanced with similar accuracies,
and a second phase (from time step ${\sim}300$) in which the two remaining groups compete until one of them dominates.
This illustrates a possible shortcoming of static fairness definitions: 
At the onset, all groups are present, but disparity in accuracy is $\approx2\%$. 
As time progresses, the system becomes more "fair" --- but only because 
the first group has been suppressed (i.e., its relative size is $\le 1\%$).
Auditing the prediction algorithm for fairness at this point in time
creates an appearance of fairness.
This holds only for the two groups that remain (disparity $\approx 0.2\%$),
and there is insufficient observed data to help identify the preceding existence of a missing group.
\Cref{app:fairness_additional_experiments} provides additional support, exploring similar phenomena under different learning algorithms.
\squeeze

\section{Discussion}

This work introduces evolutionary prediction games, a unified framework for analyzing self-selection feedback loops between learning algorithms and user populations. Our analysis reveals a gap: in idealized settings, repeated training drives competitive exclusion, whereas practical constraints enable a rich variety of long-term outcomes, including stable coexistence and mutualistic interactions.

The study of population dynamics is central to biology and ecology,
a task for which population games have proven to be highly effective.
We believe there is need to ask similar questions regarding the dynamics of social populations,
as they become increasingly susceptible to the effects of predictive learning algorithms.
The social world is of course quite distinct from the biological,
but there is nonetheless much to gain from adopting an ecological perspective.
One implication of our results is that, 
without intervention, learning might drive a population of users
to states which would otherwise be unfavorable.
Ecologists address such problems by studying and devising tools for effective conservation.
Similar ideas can be used in learning for fostering 
and sustaining heterogeneous user populations.
Our work intends to stir discussion about such ideas within the learning community.

\section*{Acknowledgements} 
The authors would like to thank 
Mor Nitzan,
Oren Kolodny,
Ariel Procaccia,
Fernando P. Santos,
Jill-J{\^e}nn Vie,
Evyatar Sabag,
and
Anonymous Reviewers
for their insightful remarks and valuable suggestions.
Nir Rosenfeld is supported by the Israel Science Foundation grant no. 278/22.
Eden Saig is supported by the Israel Council for Higher Education PBC scholarship for Ph.D. students in data science.

\bibliographystyle{plainnat}
\bibliography{citations.bib}

\appendix

\newpage

\newpage
\section{Broader Impact}

Our work seeks to shed light on the question:
how does the deployment of learned classifier impact the long-term composition of the population of its users?
Posing this question as one of equilibrium under a novel population game,
our results describe the possible evolutionary outcomes
under simple conditions (training once),
optimal conditions (perfect classifiers),
and realistic settings (namely empirical surrogate loss minimization).
The fact that learning can influence population dynamics can have major implications on social and individual outcomes.
Our work intends to illuminate this aspect of learning in social contexts,
and to provide a basic and preliminary understanding of such interactions.
\squeeze

\paragraph{Scope and limitations.}
As the bulk of our results are theoretical in nature,
naturally they 
rely on several assumptions.
Some regard the learning process:
for example, we consider only simple accuracy-maximizing algorithms
that optimize over a fixed-size training set sampled only from the current (and not past) data distribution.
Others concern dynamics:
for example,
we assume no exogenous forces (i.e., which could prevent extinction),
intra-group distribution shifts,
or inter-group dependencies (other than those indirectly formed through the classifier).
Whether our results hold also when these assumptions are relaxed is an important future question.

It is also important to interpret our results appropriately and with care.
For example, statements which establish `extinction' as a likely outcome
apply only limiting behavior, and so are silent on trajectories,
rates, or any finite timepoint.
Our results also rely on users being partitioned into non-overlapping groups.
In reality, group memberships are rarely exclusive (or even fixed),
and the dependence of individual behavior on group outcomes (in our case, marginal group accuracy---which defines fitness) is unlikely to be strict.

The above points should be considered when attempting to draw conclusions from our theoretical results to real learning problems.
However, and despite these limitations,
we believe the message we convey---which is that learning in social settings \emph{requires} intervention and conservation to ensure sustainability of social groups and individuals---%
applies broadly and beyond our framework's scope.

\paragraph{Diversity vs. homogeneity.}
Our results suggests that systems of the type we analyze will tend towards 
less diverse population compositions.
Nonetheless, and as we state throughout, our results should be taken as describing the \emph{tendency} of population dynamics --- when selective pressure is the primary force, when the learning algorithm is optimal retraining, absent any interventions, and at the limiting equilibrium. 

Of course, in reality, many systems are diverse.
Possible explanations for why a particular system at a given point in time can sustain several groups are:
\begin{itemizecompact}
\item
\textbf{Time}: dynamics have not yet reached equilibrium, or converge very slowly.

\item
\textbf{Learning}: the algorithm is not optimal (e.g., uses a proxy or finite data), or is designed to balance groups (e.g., via steering, regularization, or robust training).

\item
\textbf{Intervention}: system administrators may intentionally intervene to facilitate diversity; for example, consider how firms invest significant resources in marketing (e.g., Meta reports investing about \$2B per year\footnote{Meta Platforms, Form 10-K - Annual report, 2024, SEC Accession No. 0001326801-25-000017. “Advertising Expense”.}) to ensure a constant influx of users from diverse social groups.

\item
\textbf{Dynamics}: Other forces are at play, such as lack of alternatives, cross-group influence, rigidity of demand, non-linear heterogeneous utility, etc.
\end{itemizecompact}

Nonetheless, an important but subtle point is that when we observe a system consisting of several groups, we do not observe other groups that could have been supported, as they may have already been driven out. Thus, what we perceive as a mixed state may actually be a snapshot of a process in which “weaker” groups have already vanished, with others possibly to follow.

In biology, the question of coexistence is prevalent, despite consensus that natural selection is the primary force at play. The study of coexistence therefore focuses on how it is sustained despite natural selection. This has revealed many biological and ecological mechanisms that can facilitate coexistence, such as resource partitioning, environment variation, and stabilization via intermediates. In this work we focus on learning itself as a possible stabilizing mechanism (Sections 5,6), which we view as a first step towards the study of coexistence in this context (as also noted by Reviewer oLKx). We also believe that prediction games in particular, and evolutionary game theory more generally, can aid future research in exposing additional mechanisms that can facilitate diversity in environments with learning.

\paragraph{Practical implications.}
Evolutionary prediction games can help designers and operators of ML systems anticipate performative population shifts, and integrate this foresight into their decision-making. One key message is that applying conventional retraining risks favoring some groups over others in the long term, even if unintentionally. Evolutionary prediction games can help understand which groups are at risk, why they are at risk, and highlight possible avenues for interventions, if such are needed. One example is to decide on using a stabilizing learning algorithm (such as the one we propose in Prop. 2) instead of retraining to help steer towards desired states. Another is to target the retention of at-risk groups, either by providing benefits or reducing costs (e.g., via subsidies), in order to improve their fitness. A third is to understand if and when to increase the influx of certain groups, e.g. via marketing efforts, to balance the natural tendencies of the system. We will extend the discussion to emphasize these insights.

\section{Further Discussion of Related Work}
\label{appendix:related}
\subsection{Performative Prediction}
\label{sec:relation_to_performative}
\paragraph{Stateful performative prediction.}
Performative prediction \citep{perdomo2020performative} is a framework for modeling the effects of predictive models on the data distributions they are tested and trained on (see \citep{hardt2023performative} for a recent survey).
Dynamics induced by evolutionary prediction games are a special case of the \emph{stateful performative prediction} framework \citep{brown2022performative, ray2022decision}, which models stateful feedback loops.
The performative effects of a predictor $h\in\Hypotheses$ are formally modeled using a transition map $D_{t+1}=\mathcal{T}(D_t, h)$, where $D_t,D_{t+1}\in\DistOver{\FeaturesLabels}$ are the data distributions at times $t,t+1$, and $h\in\Hypotheses$ is the predictor being deployed. 
The fixed points of the dynamics induced by evolutionary prediction games (i.e. equilibrium states $\vp^*$), correspond to \emph{fixed point distributions} as defined in \citep[Definition 4]{brown2022performative}, as they satisfy $D_{\vp^*}=\mathcal{T}(D_{\vp^*},h_{\vp^*})$, where $h_{\vp^*}$ is the classifier trained on $D_{\vp^*}$. 
Moreover, we note that the notion of evolutionary stability is stronger than the notion of fixed point distributions, as evolutionarily stable points also remain stable under small perturbations around $\vp^*$.

\paragraph{Structural assumptions reduce effective dimensionality.}
Compared to the general performative prediction framework, natural selection dynamics induce transition maps which are lower-dimensional, and thus more amenable for analysis. The general space of feature-label distributions is infinite-dimensional under the common assumption $\Features\subseteq\Reals^d$, and the space of classifiers $\Hypotheses$ is high-dimensional when the hypotheses class $\Hypotheses$ is complex (e.g., deep neural networks). Thus, the general performative prediction framework allows for transition operators which are infinite-dimensional and notoriously challenging to analyze (e.g., as  noted by \citet{conger2024strategic}).
In contrast, in dynamics induced by evolutionary prediction games, the stateful transition operator $\mathcal{T}$ operates within a low-dimensional subspace of all possible feature label-distributions (mixtures of fixed distributions), and depends on predictors through low-dimensional statistics (their marginal accuracy on each group).

Formally, given a collection of group distributions $D_1,\dots,D_K\in\DistOver{\FeaturesLabels}$, denote the space of mixtures by $\Delta_\mathrm{mix}=\Set{\sum_k p_k D_k \mid \vp \in \DistOverGroups}$, and note that $\Delta_\mathrm{mix} \subseteq \DistOver{\FeaturesLabels}$. Any mixture $D\in\Delta_\mathrm{mix}$ can be represented by a point on the $(K-1)$-dimensional simplex $\vp\in\DistOverGroups$. Additionally, as performative response depends on %
predictor accuracy, the transition operator depends on the classifier $h$ through a $K$-dimensional vector of accuracies $\mathbf{acc}(h)=\left(\acc_1(h),\dots,\acc_K(h)\right)\in\UnitInterval^K$. Therefore, transition operators induced by the evolutionary model can be viewed as lower-dimensional mappings $\vp' = \mathcal{T}(\vp, \mathbf{acc}(h))$ from a $(2K-1)$-dimensional space to a $(K-1)$-dimensional space.

\paragraph{Further relations to performative risk and forecasting.}
The gaps between the average accuracies of equilibrium points relates to the distinction between performative optimality and performative stability \citep[e.g.,][]{miller2021outside}. Additionally, convex functions over the simplex (and their sub-gradients) emerge as a key component in our analysis of oracle classifiers (\Cref{sec:retraining_optimal_classifiers}), and also appears as a key component in the analysis of proper scoring rules and their performative counterparts \citep{oesterheld2023incentivizing, chan2022scoring} through the fundamental characterization theorem \citep{gneiting2007strictly,savage1971elicitation}. 
From this perspective, interpreting the convex potential of oracle-learner games as a generator of proper scoring rules is an intriguing direction for future work.
\squeeze

\subsection{Long-Term Fairness}
\label{sec:relation_to_fairness}
Group fairness is often formalized through measures of group disparity, such as the difference in error rates across groups (see, e.g., \citep{dwork2012fairness,barocas2023fairness}).
However, these measures may fail to capture fairness in dynamic settings. 
The long-term fairness literature therefore studies the long-term implications of population dynamics on algorithmic fairness;
one straightforward example is that if a classifier is trained to comply with fairness constraints, dynamics can cause this guarantee to erode over time \citep{liu2018delayed}
(See \citep{gohar2024long} for a recent survey).
Our work proposes a complementing (and somewhat nuanced) counterfactual perspective,
suggesting that a system can appear to comply with a fairness definition, but only because some groups were already driven out of the system, and disparity is measured only over the groups that remain.

Several works in the long-term fairness literature also consider replicator (and replicator-like) dynamics of different kinds and variations as models for user response behavior 
\citep[e.g.,][]{hashimoto2018fairness,raab2021unintended,yin2023long,dean2024emergent},
albeit not in an evolutionary game setting.
Perhaps closest to our approach is \citet{hashimoto2018fairness}, who study worst\nobreakdash-group accuracy dynamics under repeated training in a mixture population that evolves via stochastic replicator-like dynamics supplemented by a strong, balanced influx of new users that join the system regardless of its performance (parameterized by $b_k$ in their model). While their stability guarantees rely on the assumption of this strong user flow, our results offer a complementary view, showing that stable coexistence may emerge even without an exogenous user flow, under standard empirical risk minimization, and without compromising worst-group accuracy (see \Cref{sec:coexistence}).
We also extend their ERM instability results to a wider family of behavioral dynamics, and establish global convergence guarantees, rather than local (see \Cref{sec:retraining_optimal_classifiers}, and \Cref{appx:microfoundations}).

\subsection{Dynamics-Aware Learners}
Finally, our framework models the learning algorithm as an entity which is only informed by the current state of population $h\sim\Learner(\vp)$, and in particular does not take into account the implications of classifier deployment. While still common in many practical scenarios (see, e.g., \citep{dean2024accounting, kiyohara2025policy, kleinberg2024challenge}), we also note that various practical systems already optimize for long-term objectives. Examples includes learners that reduce hidden incentives to modify the population of users \citep[Sec. 7.3]{krueger2020hidden}, adaption to shifting user preferences \citep{kasirzadeh2023user, carroll2022estimating}, or directly maximizing long-term engagement via reinforcement learning (e.g., \citep{afsar2022reinforcement}). 
While our focus here is on local learning rules, analyzing dynamics-aware learners from an evolutionary perspective
is an intriguing direction for future inquiry.

\section{Microfoundations} \label{appx:microfoundations}

Natural selection dynamics appear in a variety of settings. This appendix offers a representative (though not exhaustive) set of examples showing how local, myopic user behaviors aggregate into population‐level dynamics that satisfy our core assumptions. 
Using online social platforms as a running example, we show how different behavioral patterns (e.g., invitation, imitation, resource allocation) and different types group affiliations (e.g., social, behavioral, demographic) may give rise to natural‐selection dynamics in various parts of the ecosystem. We note that similar microfoundations can be identified across a wide range of domains, such as finance, healthcare, and education.

\subsection{Reproduction}
Consider a recommendation system based on prediction. Accurate recommendations motivate users to invite peers from their social group, while the presence of alternatives creates a risk of dropout.

Formally, assume that users from $K$ distinct groups interact with a prediction system over time, and denote by $N_k$ the number of users in group $k\in[K]$. On average, users interact with the system $\lambda>0$ times per time step $t=1,2,\dots$ (e.g., three times per day), and at each interaction they receive a prediction (e.g., content relevance). If a prediction is accurate, users invite a peer from their group with probability $\alpha$. Independently, a user may leave the system with probability $\beta$ (e.g., due to friction or better alternatives). 
In expectation, the size of each group is multiplied by the following factor:
\begin{align*}
N_k^{t+1} 
&= \left(
1 + \lambda \left(\alpha \cdot \acc_k(h^t) - \beta\right)
\right) N_k^t
\end{align*}
In the continuous-time limit, users interact with the system $\lambda \mathrm d t$ times per step, and $\mathrm d t\to 0$. The expected growth approaches the differential equation:
$$
\dot N_k = \lambda (\alpha\cdot\acc_k(h)-\beta)N_k
$$
Where $\dot N_k=\tdt N_k$ is the time derivative of $N_k$. Denote $F_k=\lambda (\alpha\cdot\acc_k(h)-\beta)$, such that $\dot N_k=F_k N_k$. Note that the growth multiplier $F_k$ is time-dependent when the classifier $h$ is retrained. 

For group-level dynamics, denote by $N=\sum_k N_k$ the total number of users, and by $p_k=\tfrac{N_k}{N}$ the proportion of group $k$. The time derivative of $p_k$ is:
\begin{align*}
\dot p_k &= \fdt \frac{N_k}{N}
\\&
= \frac{\dot N_k N - \dot N N_k}{N^2}
\\&
= 
\underbrace{\frac{N_k}{N}}_{=p_k}
\bigg(
\underbrace{\frac{\dot N_k}{N_k}}_{\dot N_k = F_k N_k} - \underbrace{\frac{\dot N}{N}}_{\dot N=\sum F_k N_k}
\bigg)
\\&= p_k \left(
F_k - \sum_{k'\in[K]} p_{k'} F_{k'}
\right)
\end{align*}
Thus, the aggregate dynamics of group proportions $\vp$ satisfy the \emph{replicator equation} \citep{taylor1978evolutionary}, which satisfies our core assumptions. The fitness function in this case is $F_k(\vp)=\lambda \alpha \cdot \acc_k(h_\vp) - \lambda \beta$.

\subsection{Imitation}
Consider a media platform fostering content creators, whose profit depends on the extent to which their content reaches their target audience through accurate predictions. A successful approach (e.g., content style, advertising strategy) of one creator is likely to be mimicked by others.

Formally, suppose there a large population of content producers aiming to market some good. Each producer has a content production strategy $k$ (e.g., short versus long videos), the price of the good is $r$, and the cost of production under strategy $k$ is $c_k$ (e.g., longer videos are typically more costly to produce). A prediction algorithm observes content and makes recommendation decisions to content consumers. A prediction is considered accurate if the content is relevant to the consumer, and each successful prediction yields a conversion (e.g., sale) with probability $\alpha$. Denote by $x\sim D_k$ a content item produced by strategy $k$, and its corresponding label by $y$. The expected profit of a producer using strategy $k$ is: 
$$F_k=r \cdot \alpha \prob{D_k}{h(x)=y} - c_k$$
At each step, a content creator observes another creator's content (e.g., on social media). If their own strategy is less successful than the one they observe, they adopt the other's strategy with probability proportional to the difference in utilities. Denote the total size of the population by $N$ (here assumed to be large but fixed), the number of creators who adopt strategy $k$ by $N_k$, and the proportion of group $k$ by $p_k=N_k/N$. \citet[Chapter 4]{sandholm2010population} shows that this form of \emph{pairwise proportional imitation} converges in the large-population limit to the replicator dynamics \citep{taylor1978evolutionary}, which satisfy our core assumptions.

\subsection{Online resource allocation}
Natural‐selection dynamics between groups of users can also arise from interaction between prediction algorithms in multi‐stage pipelines.
Consider a firm that acquires users though $K$ marketing campaigns, each targeting a different demographic. 
The firm uses a classifier to personalize the experience of incoming users, and each accurate prediction yields a conversion with probability $\alpha$. Campaign $k$ therefore generates expected utility:
$$
F_k = r \cdot \alpha \cdot \acc_k(h) - c_k
$$
where $r$ is the revenue per conversion, and $c_k$ is the average cost of user acquisition for campaign $k$. 
The firm has a fixed advertising budget $B$, which it dynamically allocates across marketing campaigns. Denote the spend on campaign $k$ by $p_k B$, and assuming the firm reassigns spend according to a \emph{multiplicative weights} update rule \citep{freund1999adaptive} for dynamic allocation, we obtain:
$$p_{k}^{t+1} = p_k \frac{\beta ^ {F_k}}{\sum_{k'} \beta ^ {F_{k'}}}$$
where $\beta>1$ is a hyper-parameter of the MW algorithm. As $\beta\to1$, dynamics converge towards the replicator dynamics (cf. \Cref{eq:discrete_replicator}), which satisfies our core assumptions.

\section{Deferred Proofs}
\label{sec:proofs}

\subsection{Preliminaries}
\label{subsec:population_game_preliminaries}

\begin{figure*}
    \centering
    \includegraphics[width=\textwidth]{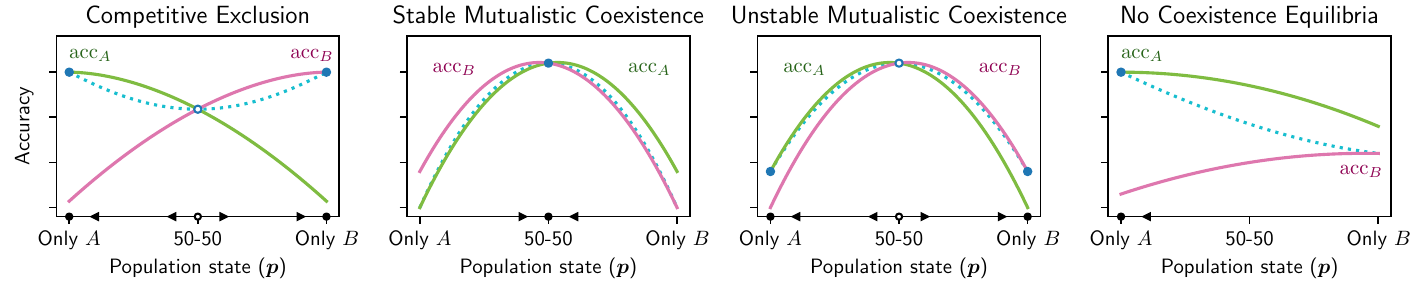}
    \vspace{-1em}
    \caption{
    Graphical representation of two-group evolutionary prediction games (\Cref{def:prediction_game}).
    Groups are denoted by $\Set{A,B}$. 
    Solid lines describe expected group accuracies for each population state $\vp$, dotted line describes population accuracy $\acc_\vp(h)$, dots indicate Nash equilibria, and x-axis arrows indicate evolutionary stability.
    Schematic plots illustrate a selection of possible scenarios:
    \textbf{(Left)}
    Competitive exclusion. There are two stable equilibria, each dominated by a single group, with locally maximal population accuracy. Additionally, there is an unstable coexistence equilibrium with lower overall accuracy.
    \textbf{(Center Left)}
    Stable mutualistic coexistence. There is a stable coexistence equilibrium, where the accuracy of both groups is maximized.
    \textbf{(Center Right)} 
    Unstable mutualisitic coexistence. Coexistence has maximal population accuracy, but it is unstable.
    \textbf{(Right)} 
    No coexistence equilibria. There is a single equilibrium dominated by a single group.
    }
\label{fig:prediction_game_types}
\end{figure*}

\paragraph{Distributions.}
The set of distribution over a set $A$ is denoted by $\DistOver{A}$, and the set of distributions over $[K]$ is denoted by $\DistOverGroups=\DistOver{[K]}=\Set{\vp\mid \sum_k p_k=1, p_k\ge 0\ \ \forall k\in[K]}$.
The support of a distribution $\vp\in\DistOverGroups$ is denoted by $\support(\vp)=\Set{k\in[K]\mid p_k>0}$.
We denote the space tangent to the $K$-simplex by $T\DistOverGroups=\Set{z\in\Reals^K\mid \sum_k z_k=0}$.
We denote $\vone=\left(1,\dots,1\right)\in\Reals^K$, and denote by $\mPhi=\mI-\vone\vone^T$ the orthogonal projection matrix from $\Reals^K$ to $T \DistOverGroups$. 
For $k\in[K]$, we denote by $\SingletonDistribution_k\in\DistOverGroups$ the singleton categorical distribution supported over $k$, representing a mixture which only contains the distribution $D_k$.

\paragraph{Supervised learning.}
We denote the space of features by $\Features$, the space of labels by $\Labels$, the hypotheses class by $\Hypotheses\subseteq\Labels^\Features$, and the loss function by $\Loss:\Labels^2\to\Reals$.
The expected loss of hypothesis $h\in\Hypotheses$ under population distribution $D\in\DistOver{\FeaturesLabels}$ is denoted by $\Loss_D(h)=\expect{(x,y)\sim D}{\Loss(h(x),y)}$. A supervised learning algorithm $\Learner$ is a possibly-stochastic mapping from a distribution $D\in\DistOver{\FeaturesLabels}$ to a classifier $h\sim\Learner(D)$, formally $\Learner:\DistOver{\FeaturesLabels}\to\DistOver{\Hypotheses}$.

\paragraph{Population state.}
Each group $k$ is associated with a feature-label distribution $D_k\in\DistOver{\FeaturesLabels}$.
When the population state is $\vp\in\DistOverGroups$, the training state is a mixture distribution denoted by $D_\vp=\sum_{k\in[K]} p_k D_k$. 
We use a vector subscript notation to distinguish between group distributions (e.g., $D_k$) and mixture distributions (e.g., $D_\vp$).
Given hypothesis $h\in\Hypotheses$, the expected loss of group $k$ is denoted by $\Loss_k(h)=\expect{(x,y)\sim D_k}{\Loss(h(x),y)}$, and the vector of expected group losses is denoted by $\LossVec(h)=\left(\Loss_1(h),\dots,\Loss_K(h)\right)$.
For $\vp\in\DistOverGroups$, the expected loss of $h$ under distribution $D_\vp$ is denoted by $\Loss_\vp(h) = \expect{(x,y)\sim D_\vp}{\Loss(h(x),y)}$.

\paragraph{Population games.}
We follow the notations of \citet{sandholm2010population}, and include the key definitions here for completeness. 
Population games associate each group $k\in[K]$ with a fitness function $F_k(\vp)$, which maps a population state $\vp\in\DistOverGroups$ to the evolutionary fitness of group $k$. 
A Nash Equilibrium (NE) of a population game is a population state $\vp^*$ which satisfies $\support(\vp^*)\subseteq\argmax_{k\in[K]} F_k(\vp^*)$, and every population game admits at least one Nash equilibrium \citep[Theorem 2.1.1]{sandholm2010population}. 
A Restricted Equilibrium (RE) of a population game is a population state $\vp^*$ which satisfies $\support(\vp^*)\subseteq\argmax_{k\in\support(\vp^*)} F_k(\vp^*)$, or equivalently $F_{k}(\vp^*)=F_{k'}(\vp^*)$ for all $k,k'\in\support(\vp^*)$. 
A population game is a potential game if it can be represented as a gradient of some function. Formally, $F:\DistOverGroups\to\Reals^K$ is a potential game if there exists $f:\DistOverGroups\to\Reals$ such that $F=\nabla f$, where $\nabla f$ is the gradient of $f$ over the simplex (see \citep{sandholm2010population}, Section 3.2 for rigorous definitions).

\paragraph{Competition, mutualism, and evolutionary stability.} 
Coexistence equilibria are fixed points in which two groups or more persist $\Size{\support(\vp^*)}\ge2$. We say that coexistence is \emph{competitive} if each group's accuracy at $\vp^*$ is lower than its accuracy in isolation (i.e., when it is the sole group in the population). Conversely, we say that interactions are \emph{mutualistic} if every group’s accuracy at $\vp^*$ exceeds the accuracy it obtains in isolation, indicating that coexistence benefits all groups (see \Cref{fig:prediction_game_types} for an illustration). Additionally, we say an equilibrium is \emph{evolutionarily stable} if it attracts nearby population states (see \citep[Section 8.3]{sandholm2010population} for formal definitions).

\subsection{Evolutionary Dynamics}
\label{subsec:appendix_dynamics}

\begin{figure*}
    \centering
    \includegraphics[width=0.9\textwidth]{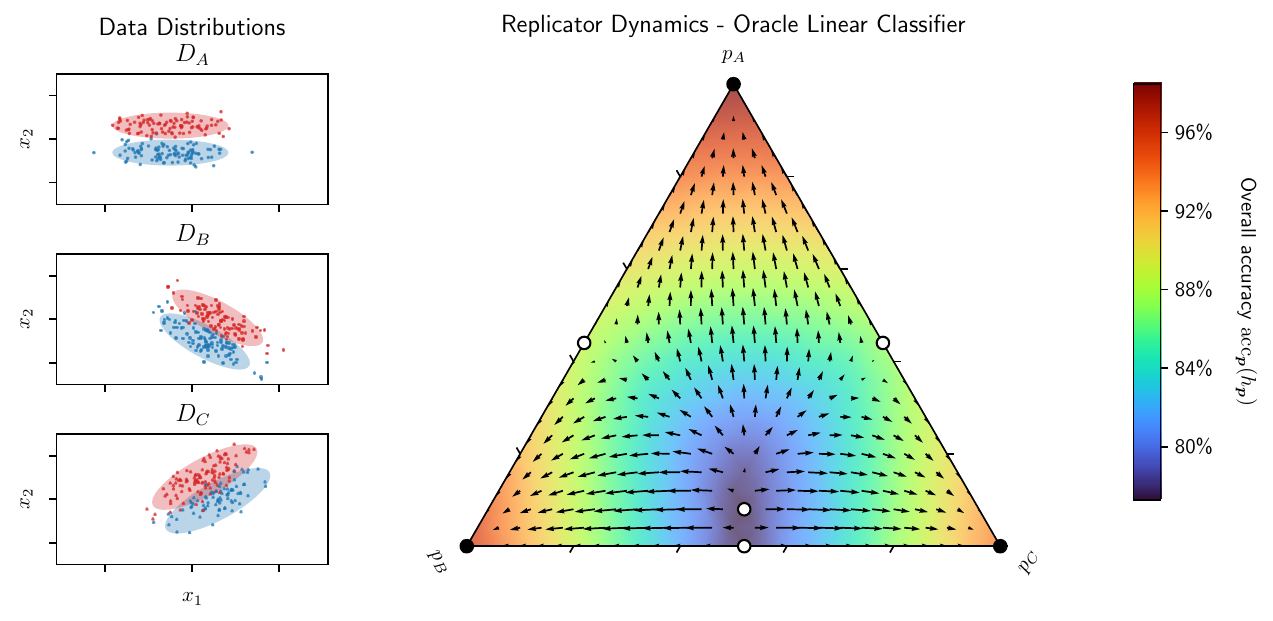}
    \vspace{-1em}
    \caption{
    Evolutionary dynamics of an oracle classifier in a three-group setting.
    \textbf{(Left)} Data distributions of groups $\Set{A,B,C}$.
    \textbf{(Right)} Phase diagram of replicator dynamics over the probability simplex, induced by the evolutionary prediction game. Background color represents the overall accuracy $\acc_\vp(h_\vp)$, vectors represent the flow of replicator dynamics, and markers represent fixed points, with marker color indicating stability. The dynamics have 7 fixed points: One unstable fixed point at the interior of the simplex (unstable coexistence equilibrium), three stable fixed points at the vertices of the simplex (stable equilibria with single-group dominance), and three additional unstable points at the boundaries (unstable restricted equilibria induced by the replicator dynamics).
    }
\label{fig:ternary}
\end{figure*}

For our theoretical analysis, we assume that $\vp$ evolves in time according to an ordinary differential equation of the form $\dot{\vp} = V_F(\vp)$, where $V_F:\DistOverGroups\to T\DistOverGroups$ is a vector field tangent to the simplex, induced by the fitness function $F(\vp)$, assumed to be continuous (see \citep[Theorem 4.4.1]{sandholm2010population}). 
A common example for such dynamical system is the replicator equation $\dot p_k = p_k (F_k(\vp)-\sum_{k'} F_{k'}(\vp)) $, which corresponds to the tangent vector field $V_F(\vp)=\vp \odot \left(F(\vp)-\sum_{k'} p_{k'} F_{k'}(\vp) \vone \right)$, where $\odot$ is the Hadamard product (see \Cref{fig:ternary}). The replicator dynamics are a common model of evolutionary dynamics (see \Cref{appx:microfoundations}, and e.g. \citep{traulsen2006coevolutionary,kleinberg2011beyond}), and satisfies our core assumptions. Moreover, we also emphasize that our analysis applies to any $V_F$ which satisfies the core assumptions below.

With respect to $V_F$, we make the following assumptions:
\begin{itemizecompact}
\item \textbf{Continuity:} We assume that $V_F(\vp)$ is continuous with respect to $\vp$.
\item \textbf{Positive correlation:} We assume that $V_F(\vp)$ is directionally aligned with the fitness vector $F(\vp)$ whenever a population is not at rest, formally $V_F(\vp)\cdot F(\vp) > 0$ whenever $V_F(\vp)\neq\vzero$. For a detailed discussion, see \citep[Section 5.2]{sandholm2010population}.
\end{itemizecompact}
For correspondence to game equilibria, we assume either of the following assumptions:
\begin{itemizecompact}
\item \textbf{Nash stationarity:} Under Nash stationarity, all fixed points $V_F(\vp^*)=\vzero$ correspond to Nash equilibria $\vp^*$ of the game $F$. See \cref{eq:nash_equilibrium} for the definition of Nash equilibrium, and \citep[Section 5]{sandholm2010population} for a detailed discussion.
\item \textbf{Imitative dynamics:} Under imitative dynamics, agents adopt strategies based on observing and copying the behaviors of others in a population. See \citep[Section 5.4]{sandholm2010population} for a formal definition using the revision protocol formalism.
\end{itemizecompact}

The replicator equation introduced in \Cref{sec:prediction_games} is a special case of imitative dynamics, which satisfy our assumptions. 
In contrast to dynamics with Nash stationarity, imitative dynamics admit additional rest points at restricted equilibria of the population game (formally, states satisfying $F_k(\vp^*)=F_{k'}(\vp^*)$ for all $k,k'\in\support(\vp^*)$). However, we note that this does not affect our main results (\Cref{thm:optimal_learner} and its extensions, and our constructions in \Cref{sec:coexistence}), as any Nash equilibrium is also a restricted equilibrium, and any non-Nash restricted equilibrium corresponds to an unstable fixed point of the imitative dynamics (see \citep[Section 8.1]{sandholm2010population}). See \Cref{fig:ternary} for an illustration of Nash and restricted equilibria in a three-group setting.

\paragraph{Time scales.} Rates of convergence can be tuned by scaling the vector field $V_F$. For example, replacing $V_F$ with $\alpha V_F$ for $\alpha>0$ compresses or stretches time by a factor of $\alpha$ while still maintaining our core assumptions. In particular, if $\vp^t$ satisfies $\dot \vp=V_F(\vp)$, then $\vq^t=\vp^{\alpha t}$ is a solution to $\dot\vq=\alpha V_F(\vq)$ with the same initial conditions. Thus, our convergence and stability results hold irrespective of whether the dynamics evolve quickly or slowly.

\paragraph{Dynamics induced by potential games.} Dynamics induced by potential games are closely related to the structure of the potential function $f(\vp)$. We informally state key results that are relevant to our setting:
By \citep[Theorem 3.1.3]{sandholm2010population}, a state $\vp^*$ is a Nash equilibrium of $F$ 
if and only if it is an extremum of the potential function $f$. 
By \citep[Lemma 7.1.1]{sandholm2010population} potential function are Lyapunov functions, and by \citep[Theorem 7.1.2]{sandholm2010population}, dynamics satisfying the positive correlation property converge towards Nash equilibria (restricted Nash equilibria in case of replicator dynamics). 
Finally, by \citep[Theorem 8.2.1]{sandholm2010population}, a population state $\vp^*$ is a stable equilbrium if and only if it is a local maximizer of $f(\vp)$.

\subsection{Basic Claims}
\subsubsection{Linearity}
\begin{proposition}[Expected loss is linear in $\vp$ for a fixed hypothesis]
\label{claim:fixed_h_loss_convex_combination}
For any $h\in\Hypotheses$ and $\vp\in\DistOverGroups$, The expected loss can be represented as a convex combination:
$$
\Loss_\vp(h) = \sum_{k\in[K]} p_k \cdot \Loss_k(h)
$$
\end{proposition}
\begin{proof}
Given a population state $\vp$, treat it as a categorical random variable over population groups, denote by $k \sim \vp$ the group from which a feature-label pair gets selected. Applying the law of total expectation and using the definition of $\Loss_\vp(h)$, we obtain:
\begin{equation*}
\begin{aligned}
\Loss_\vp(h)
&= \expect{(x,y)\sim D_\vp}{\Loss(h(x),y)} \\
&= \expect{k \sim \vp}{\expect{(x,y)\sim D_k}{\Loss(h(x),y)}} \\
&= \sum_{k\in[K]} p_k \cdot \expect{(x,y)\sim D_k}{\Loss(h(x),y)} \\
&= \sum_{k\in[K]} p_k \cdot \Loss_k(h)
\end{aligned}
\end{equation*}
\end{proof}

\subsubsection{Overall Accuracy Equality}
\label{sec:equilibria_are_fair}
Equilibria of evolutionary prediction games also satisfy a natural fairness criterion. A classifier $h$ satisfies \emph{overall accuracy equality} if all groups have equal prediction accuracy \citep{verma2018fairness}. In the context of natural selection, we consider this with respect to the groups currently present in the population:

\begin{definition}[Overall accuracy equality; {\citep[e.g.,][]{verma2018fairness}}]
Let $h:\Features\to\Labels$, and let $\vp\in\DistOverGroups$. $h$ satisfies overall prediction equality if $\acc_k(h)=\acc_{k'}(h)$ for all $k,k'\in\support(\vp)$.
\end{definition}

Under retraining, a different classifier $h\sim\Learner(\vp)$ is deployed at each time step. It is therefore natural to define this property with respect to the learning algorithm, rather than a single classifier. 
In the following definition, we take the expectation over the stochasticity of the sampling and learning process to representing the average over time under retraining:

\begin{definition}[Overall accuracy equality in expectation]
Let $\vp$ be a population state, and let $\Learner(\vp)$ be a learning algorithm. $\Learner$ satisfies overall prediction equality in expectation if $\expect{h\sim\Learner(\vp)}{\acc_k(h)}=\expect{h\sim\Learner(\vp)}{\acc_{k'}(h)}$ for all $k,k'\in\support(\vp)$.
\end{definition}

\begin{proof}[Proof of \Cref{prop:equilibrium_fairness}]
Let $\vp^*$ be a Nash equilibrium of the evolutionary prediction game induced by $\Learner(\vp)$. By eq.~(\ref{eq:nash_equilibrium}), it holds that $\support(\vp^*)\subseteq \argmax_k F_k(\vp^*)$, and therefore there exists some $a\in\Reals$ such that $F_k(\vp^*)=a$ for all $k\in\support(\vp^*)$. Then by the definition of the evolutionary prediction game (Def.~\ref{def:prediction_game}), it holds that:
$
\expect{h\sim\Learner(\vp)}{\acc_k(h)}=\expect{h\sim\Learner(D)}{\acc_{k'}(h)}=a
$
for all $k,k'\in\support(\vp^*)$.
\end{proof}

Furthermore, we note that similar reasoning can be utilized to prove an analogous correspondence between overall accuracy fairness and restricted equilibria of the game.

\subsection{Training Once}
\label{sec:training_once_proofs}

We begin with a characterization of evolutionary dynamics for a simple setting in which the classifier is trained only once on the initial distribution, and henceforth kept fixed.
This is useful as a first step since by fixing $h=h^0$,
outcomes can be attributed solely to the induced changes in the population.

We will prove the following theorem:
\begin{theorem} \label{prop:train_once}
Let $\Learner$ be a learning algorithm.
Denote the initial population state by $\vp^0$, and the initial classifier by $h^0\sim\Learner(\vp^0)$.
If $h^0$ is deployed at every time step (i.e., irrespective of the 
population state $\vp$), then
it holds that:
\begin{enumeratecompact}
    \item \textbf{Competitive exclusion:} $\vp^*$ is supported on $k^* \in \argmax_k \acc_{k}(h^0)$, i.e., initially-fittest groups.
    \item \textbf{Stability:} 
    $\vp^*$ is stable.
    \item \textbf{Accuracy:} 
    $
    \tfrac{\mathrm d}{\mathrm dt}
    \acc_{\vp^t}(h^0) \ge 0
    $, i.e., overall accuracy weakly improves over time.
\end{enumeratecompact}
\end{theorem}

\begin{definition}[Train-once prediction game]
Let $\vp^0\in\DistOverGroups$ be the population state at time 0, and denote by $h^0\in\Hypotheses$ the hypothesis learned at time 0. 
The train-once prediction game is a population game with the following fitness function:
$$
F_k(\vp;h^0) = -\Loss_k(h^0)
$$
\end{definition}

\begin{definition}[Constant game; {\citep[Section 3.2.4]{sandholm2010population}}]
\label{def:constant_game}
A population game $F:\DistOverGroups\to\Reals^K$ is a constant game 
if there exist constants $\alpha_1,\dots,\alpha_K\in\Reals$ such that $F_k(\vp)=\alpha_k$ for all $\vp\in\DistOverGroups$.
\end{definition}

\begin{proposition}
\label{prop:train_once_general_loss}
let $\Loss$ be a loss function, let $\vp^0\in\DistOverGroups$ be the population state at time 0, and let $h^0\in\Hypotheses$ be the hypothesis learned at time 0.
Denote by $F_k(\vp)$ the corresponding train-once prediction game, and denote by $\vp^*$ a Nash equilibrium of $F$. It holds that:
\begin{enumerate}
    \item $\support(\vp^*) \subseteq \argmin_{k\in[K]} \Loss_k (h^0)$
    \item 
    $\Loss_{\vp^*}(h) \le \Loss_{\vp^0}(h)$
    \item $\vp^*$ is stable.
\end{enumerate}
\end{proposition}
\begin{proof}
Fix $h_0\in\Labels^\Features$.
In the train-once setting, the expected marginal accuracies are constant as a function of $\vp$, and therefore $F_k$ is a constant game by \cref{def:constant_game}.
By \citep[Proposition 3.2.13]{sandholm2010population},
$F_k$ a potential game with linear potential $f(\vp)=-\LossVec(h^0)\cdot\vp$.
By \citep[Theorem 3.1.3]{sandholm2010population}, the set of Nash equilibria is the set of local maximizers of $f(\vp)$, which is a convex combination over the set $\argmin_{k\in[K]} \Loss_k(h^0)$ due to linearity. Hence: 
$$\support(\vp^*)\subseteq\argmin_{k\in[K]}\Loss_k(h^0)$$
satisfying (1).

At equilibrium $\vp^*$, it holds that $\Loss_{\vp^*}=\min_{k\in[K]} \Loss_k(h^0)$. Apply \cref{claim:fixed_h_loss_convex_combination} to obtain:
\begin{equation*}
\begin{aligned}
\Loss_{\vp^0}(h^0) 
&= \sum_k p_k^0 \Loss_k(h^0) 
\\&\ge \sum_k p_k^0 \min_{k'\in[K]}\Loss_{k'}(h^0)
\\&= \min_{k'\in[K]}\Loss_{k'}(h^0)
\\&= \Loss_{\vp^*}(h^0)
\end{aligned}
\end{equation*}
satisfying (2).

Finally, observe that each $\vp^*$ is stable as a maximizer of the linear potential function \citep[Theorem 8.2.1]{sandholm2010population}.
\end{proof}

\begin{proof}[Proof of \Cref{prop:train_once}]
\Cref{prop:train_once} is a special case of \Cref{prop:train_once_general_loss} for the loss function: 
$
\Loss_k(h)
=-\acc_k(h)
$.
\end{proof}

\subsection{Retraining}
\label{sec:retraining_proofs}

\begin{figure*}
    \centering
    \includegraphics[width=0.65\textwidth]{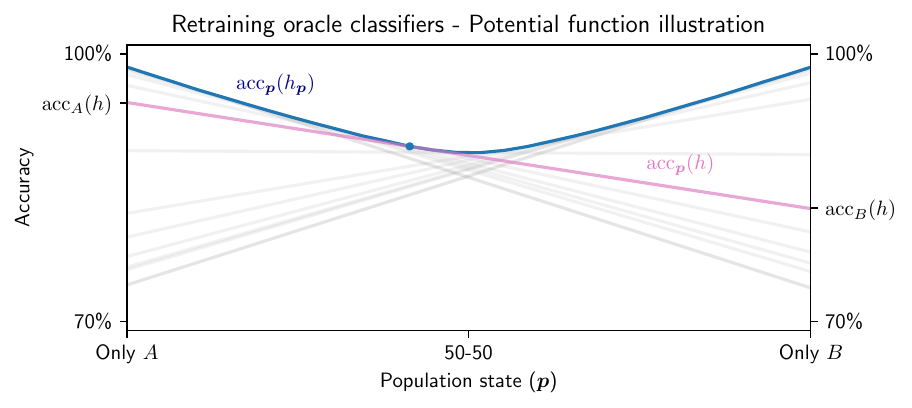}
    \vspace{-1em}
    \caption{
    Illustration of the proof of \Cref{thm:optimal_learner}. 
    Pink line: The expected accuracy of a fixed classifier $h$ is linear in $\vp$ (\Cref{claim:fixed_h_loss_convex_combination}). 
    Blue line: The overall accuracy of the oracle classifier $h_\vp$ is a pointwise maximum over linear functions, and is therefore convex (\Cref{claim:optimal_loss_is_concave}). The gradient of $\acc_\vp(h_\vp)$ over the simplex is the vector of marginal accuracies, showing that the overall accuracy is a potential function for the game (\Cref{claim:projected_subgradient_of_optimal_loss_fuction}). Properties of the game's equilibria then follow from convexity arguments.
    }
\label{fig:optimal_classifier_potential}
\end{figure*}

\begin{definition}[Retraining prediction game with an oracle predictor]
Let $\vp\in\DistOverGroups$, let $\Loss$ be a loss function, and let $\Hypotheses$ be a hypothesis class. Denote $h_\vp \in \argmin_{h\in\Hypotheses} \Loss_\vp(h)$, and assume ties are broken in a consistent way. The retraining prediction game with an optimal predictor is a population game with the following payoff function:
$$
F_k(\vp) = -\Loss_k(h_\vp)
$$

\end{definition}

\begin{definition}[Optimal loss function $l^*(\vp)$]
\label{def:optimal_loss_function}
Given hypotheses $\Hypotheses$ and population state $\vp\in\DistOverGroups$, the optimal loss at $\vp$ is:
$$
l^*(\vp) = \min_{h\in\Hypotheses} \Loss_\vp(h) 
$$
\end{definition}

\begin{proposition}[$l^*(\vp)$ is concave]
\label{claim:optimal_loss_is_concave}
The optimal loss $l^*(\vp)$
is a concave function of $\vp$.
\end{proposition}
\begin{proof}
By definition of $l^*(\vp)$:
$$
l^*(\vp)
= \min_{h\in\Hypotheses} \Loss_\vp (h)
$$
From \cref{claim:fixed_h_loss_convex_combination}, each $\Loss_\vp(h)$ is linear in $\vp$ for all $h\in\Hypotheses$, and in particular concave. 
Hence, $l^*(\vp)$ is a point-wise minimum of concave functions, and is therefore concave.
\end{proof}

\begin{definition}[Subgradient; {\citep[e.g.,][Section 23]{rockafellar1970convex}}]
\label{def:convex_function_tangent_plane}
Let $f:\DistOverGroups\to\Reals$ be a convex function over the simplex. 
The vector $\vv\in T\DistOverGroups$ is a subgradient of $f$ at point $\vp$ if for all $\vp'$ it holds that:
$$
f(\vp) + \vv \cdot (\vp'-\vp) \le f(\vp')
$$
Where $f(\vp) + \vv \cdot (\vp'-\vp)$ is the tangent plane at point $\vp$.
\end{definition}
    
\begin{definition}[Subdifferential]
Let $f:\DistOverGroups\to\Reals$ be a convex function over the simplex. 
The subdifferential of $f$ at $\vp$ is set of all subgradients at point $\vp$, denoted by $\partial f(\vp)$.
\end{definition}

\begin{proposition}
\label{claim:projected_subgradient_of_optimal_loss_fuction}
Let $h_\vp \in \argmin_{h\in\Hypotheses} \Loss_\vp(h)$ be an optimal predictor. 
The projection of the loss vector of $h$ at point $\vp$ is a subgradient of $-l^*(\vp)$:
$$
\mPhi\LossVec(h_\vp) \in \partial \left(-l^*(\vp)\right)
$$
\end{proposition}
\begin{proof}
Let $h_\vp \in \argmin_{h\in\Hypotheses} \Loss_\vp(h)$, and let $\vp,\vp'\in\DistOverGroups$.
We need to show that $-l^*(\vp)+\LossVec(h_\vp)\cdot(\vp'-\vp)\le -l^*(\vp')$.

By \cref{claim:fixed_h_loss_convex_combination}, it holds that 
$\Loss_{\vp'}(h_\vp)= \sum_k p'_k \Loss_k(h_\vp)$. By algebraic manipulation we obtain:
\begin{equation}
\label{eq:subgradient_linear_loss_as_tangent_plane}
\begin{aligned}
\Loss_{\vp'}(h_\vp)
&= \LossVec(h_\vp)\cdot\vp'
\\
&= 
\overbrace{
\underbrace{\LossVec(h_\vp)\cdot\vp}_{=l^*(\vp)}
-\LossVec(h_\vp)\cdot\vp 
}^{=0}
+ \LossVec(h_\vp)\cdot\vp'
\\
&=l^*(\vp)
+ \LossVec(h_\vp)\cdot\left(\vp'-\vp\right)
\end{aligned}
\end{equation}
As $\vp'-\vp\in T\DistOverGroups$, it is not affected by the orthogonal projection $\mPhi:\Reals^K\to T\DistOverGroups$, it holds that:
\begin{equation}
\label{eq:subgradient_projection_does_not_affect_difference}
\mPhi\left(\vp'-\vp\right)=\vp'-\vp
\end{equation}
Since $\mPhi$ is an orthogonal projection, it holds that $\mPhi^T=\mPhi$. In addition, for any vector $\vv$ it holds that:
\begin{equation}
\label{eq:subgradient_projection_transpose}
\vv^T\mPhi^T=\left(\mPhi\vv\right)^T
\end{equation}
Combining equations (\ref{eq:subgradient_linear_loss_as_tangent_plane}, \ref{eq:subgradient_projection_does_not_affect_difference}, \ref{eq:subgradient_projection_transpose}), we obtain:
$$
\mPhi \Loss_{\vp'}(h_\vp) = 
l^*(\vp)
+ \mPhi \LossVec(h_\vp)\cdot\left(\vp'-\vp\right)
$$
By \cref{claim:optimal_loss_is_concave}, it holds that $\Loss_{\vp'}(h_\vp) \ge l^*(\vp')$, and therefore:
$$
l^*(\vp)
+ \mPhi \LossVec(h_\vp)\cdot\left(\vp'-\vp\right)
\ge
l^*(\vp')
$$
$l^*(\vp)$ is concave, and therefore its negation is convex.
Multiplying both sides by $-1$, we obtain that $-\mPhi\LossVec(\vp)$ is a subgradient of $-l^*$ at point $\vp$, as required.
\end{proof}

\begin{proposition}
\label{claim:gradient_of_optimal_loss_function}
Assume $l^*(\vp)$ is differentiable, and let $h_\vp\in\argmin_{h\in\Hypotheses}\Loss_\vp(h)$.
The gradient of $-l^*(\vp)$ satisfies:
$$
\nabla \left(-l^*(\vp)\right) = -\mPhi \LossVec(h_\vp)
$$
\end{proposition}
\begin{proof}
Any convex differentiable function $f$ satisfies:
$$
\partial f(\vp) = \Set{\nabla f(\vp)}
$$
By \cref{claim:projected_subgradient_of_optimal_loss_fuction}, it holds that $-\mPhi\LossVec(h_\vp)\in\partial\left(-l^*(p)\right)$, and therefore $\nabla \left(-l^*(\vp)\right) = -\mPhi\LossVec(h_\vp)$.
\end{proof}

\begin{lemma}
\label{lemma:optimal_predictor_game_is_potential}
Let $\Hypotheses\subseteq \Labels^\Features$ and let $\Loss:\Labels^2\to\Reals$. For $\vp\in\DistOverGroups$, denote $h_\vp=\argmin_{h\in\Hypotheses} \Loss_\vp(h)$. If $\Loss_\vp(h_\vp)$ is a continuously differentiable function of $\vp$, then game $F_k(\vp)=-\Loss_k(h_\vp)$ is a potential game.
\end{lemma}
\begin{proof}
Take the function $f(\vp)=-l^*(\vp)$. By \cref{claim:gradient_of_optimal_loss_function}, the gradient of $f(\vp)$ satisfies:
$$
\nabla f(\vp)=-\mPhi\LossVec(h_\vp)=F(\vp)
$$
And therefore $F$ is a potential game.
\end{proof}

We say that a set of groups is \emph{oracle-equivalent} if they share a common oracle classifier:
\begin{definition}[Oracle-equivalent groups]
\label{def:identical_optimality}
A set of groups $K^*\subseteq K$ is \emph{oracle-equivalent} if there exists
$h\in\bigcap_{k\in K^*}\argmin_{h\in\Hypotheses} \Loss_{D_{k}}(h)$
which satisfies $\Loss_{k_1}(h)=\Loss_{k_2}(h)$ for all $k_1,k_2\in K'$.
\end{definition}

\begin{proposition}
\label{prop:identical_optimality_constant_optimal_loss}
Two groups $k,k'\in [K]$ are oracle-equivalent if any only if the expected optimal loss of their convex combination $l^*\left((1-\alpha)\SingletonDistribution_{k}+\alpha \SingletonDistribution_{k'}\right)$ is constant for all $\alpha\in[0,1]$.
\end{proposition}
\begin{proof}
If the groups are oracle-equivalent, then by \Cref{def:identical_optimality} there exists
$$
h^*\in
\argmin_{h\in\Hypotheses} \Loss_{D_{k}}(h)
\cap
\argmin_{h\in\Hypotheses} \Loss_{D_{k'}}(h)
$$ 
which satisfies
$
    \Loss_{D_{k}}(h^*)
    =
    \Loss_{D_{k'}}(h^*)
    =
    \const
    $ for some constant.
Since $h^*$ is optimal for each group separately, it is also optimal for any convex combination of the groups $(1-\alpha)D_k+\alpha D_{k'}$, and the optimal loss therefore satisfies:
$$
l^*\left((1-\alpha)\SingletonDistribution_k+\alpha \SingletonDistribution_{k'}\right)
= 
(1-\alpha)\Loss_{D_{k}}(h^*)
+
\alpha\Loss_{D_{k'}}(h^*)
=\mathrm{const}
$$

Conversely, assume that $k,k'$ are not oracle-equivalent. We consider two cases:

\begin{enumeratecompact}
    \item 
If $
\argmin_{h\in\Hypotheses} \Loss_{D_{k}}(h)
\cap
\argmin_{h\in\Hypotheses} \Loss_{D_{k'}}(h)
= \emptyset
$,
then let $\vp=0.5 \SingletonDistribution_k + 0.5 \SingletonDistribution_{k'}$. By linearity of expectation, for any $h\in\Hypotheses$ it holds that:
$$
\Loss_{D_\vp}(h)=0.5 \Loss_{D_{k}}(h) + 0.5 \Loss_{D_{k'}}(h)
$$
By the assumption, no $h\in\Hypotheses$ simultaneously minimizes both terms. Therefore, the optimal loss at $\vp$ must satisfy either $l^*(\vp) > l^*(\SingletonDistribution_k)$ or $l^*(\vp) > l^*(\SingletonDistribution_{k'})$, and therefore $l^*\left((1-\alpha)\SingletonDistribution_{k}+\alpha \SingletonDistribution_{k'}\right)$ is not constant.

\item
Otherwise, if there exists a common optimizer $h^*\in
\argmin_{h\in\Hypotheses} \Loss_{D_{k}}(h)
\cap
\argmin_{h\in\Hypotheses} \Loss_{D_{k'}}(h)
$ but $
    \Loss_{D_{k}}(h^*)
    \neq
    \Loss_{D_{k'}}(h^*)
    $, then $l^*\left((1-\alpha)\SingletonDistribution_{k}+\alpha \SingletonDistribution_{k'}\right)$ is not constant when comparing the values corresponding to $\alpha=0$ and $\alpha=1$.
\end{enumeratecompact}
\end{proof}

\begin{proposition}
\label{prop:oracle_equivalent_optimal_loss_minimizers}
The local minimizers of the optimal expected loss function $l^*$ are convex combinations of oracle-equivalent groups.
\end{proposition}
\begin{proof}
By \Cref{claim:optimal_loss_is_concave}, the optimal loss $l^*$ is concave over the simplex $\DistOverGroups$, and therefore its minimizers are convex combinations of simplex vertices for which the function is constant.
By \Cref{prop:identical_optimality_constant_optimal_loss}, such combinations correspond to sets of oracle-equivalent groups.
\end{proof}

\begin{proposition}
\label{prop:distinct_optimal_classifier_vertex_extrema}
When each group has a distinct optimal classifier, the optimal loss function $l^*(\vp)$ only has local minimizers at vertices of the simplex.
\end{proposition}
\begin{proof}
When each group has a distinct optimal classifier, the only sets of oracle-equivalent groups are the singleton sets $\Set{1},\dots,\Set{K}$. The claim then follows from \Cref{prop:oracle_equivalent_optimal_loss_minimizers}.
\end{proof}

\begin{proof}[Proof of \Cref{thm:optimal_learner}]
By applying \Cref{lemma:optimal_predictor_game_is_potential} on the loss function $\Loss_k(h)=-\acc_k(h)$, we obtain that $F(\vp)$ is a potential game with potential function $f(\vp)=\acc_\vp(h_\vp)$.
By \Cref{claim:optimal_loss_is_concave}, we obtain that $f(\vp)$ is convex over the simplex.
When no two groups share an optimal classifier \Cref{prop:distinct_optimal_classifier_vertex_extrema} holds, and local maximizers of $f$ exist and are located at the vertices of the simplex.
By \citep[Theorem 8.2.1]{sandholm2010population}, maximizers of $f(\vp)$ correspond to evolutionarily stable Nash equilibria of $F(\vp)$, and therefore any stable equilibrium $\vp^*$ satisfies $\Size{\support(\vp^*)}=1$. Additionally, we note that coexistence equilibria may exist as $f(\vp)$ may have local minimizers. This proves claims 2,3,4 (Stability, Exclusion, Coexistence) in the statement of the theorem. Finally, for claim 1 (Accuracy), by \citep[Lemma 7.1.1]{sandholm2010population} the potential function $f(\vp)$ is a Lyapunov function for any evolutionary dynamic which satisfies the positive correlation property, and therefore 
$
\tfrac{\mathrm d}{\mathrm dt} f(\vp)
=
\tfrac{\mathrm d}{\mathrm dt}
\acc_{\vp}(h_{\vp}) \ge 0
$.
\end{proof}

\subsubsection{Equivalent groups}
\label{subsec:appendix_equivalent_groups}

\begin{figure*}
    \centering
    \includegraphics[width=0.9\textwidth]{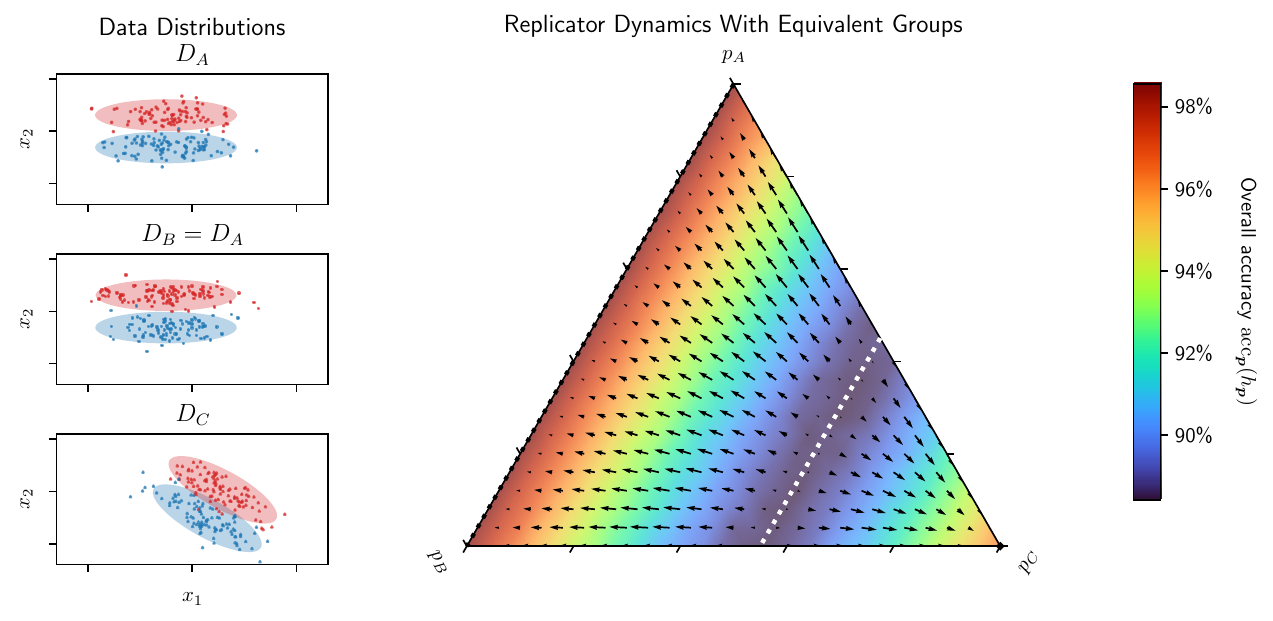}
    \vspace{-1em}
    \caption{
    Evolutionary dynamics of an oracle classifier in a three-group setting with equivalent groups (\Cref{subsec:appendix_equivalent_groups}).
    \textbf{(Left)} Data distributions of groups $\Set{A,B,C}$, where $D_B=D_A$. Groups $A,B$ are oracle-equivalent, and group $C$ is not oracle-equivalent to $A$ or $B$  (\Cref{def:identical_optimality}).
    \textbf{(Right)} Phase diagram of the corresponding replicator dynamics over the probability simplex (cf. \Cref{fig:ternary}). The dynamics has three \emph{sets} of fixed points, corresponding to \Cref{thm:optimal_learner_equivalent_groups}: The convex hull of the singleton distributions corresponding to $A,B$ (stable), the singleton distribution of group $C$ (stable),
    and the convex hull of $\vp\approx(0.45,0,0.55)$ and $\vp'\approx(0,0.45,0.55)$ (unstable).
    }
\label{fig:ternary_equivalent}
\end{figure*}

Share the same optimal classifier and achieve the same optimal accuracy (formally satisfying \Cref{def:identical_optimality}), then an extended version of \Cref{thm:optimal_learner} holds:

\begin{theorem}[Oracle classifier with equivalent groups; Extension of \Cref{thm:optimal_learner}]
\label{thm:optimal_learner_equivalent_groups}
Let $\Hypotheses$ be a hypothesis class, and denote by $\Learner^\opt$ the oracle learning algorithm with respect to $\Hypotheses$.
Assume that
at each population state $\vp$
the oracle classifier $h_\vp\sim\Learner^\opt(D_\vp)$ is deployed. 
Then it holds that:
\begin{enumeratecompact}
    \item \textbf{Accuracy:} 
    Overall accuracy increases over time,
    $
    \tfrac{\mathrm d}{\mathrm dt}
    \acc_{\vp}(h_{\vp}) \ge 0
    $.
    \item \textbf{Stability:} A stable equilibrium always exists,
    and there can be multiple such equilibria.
    \item \textbf{Competitive exclusion:} 
    For all stable equilibria,
    $\support(\vp^*)$ is oracle-equivalent (Def.~\ref{def:identical_optimality}).%
    \item \textbf{Coexistence:}
    Equilibria with non-equivalent groups may exist,
    but are unstable.

\end{enumeratecompact}
\end{theorem}

\begin{proof}
In analogy to the proof of \Cref{thm:optimal_learner}. The potential function remains $f(\vp)=\acc_p(h_\vp)$, and so does the relation between extrema of the potential and equilibria of the dynamics. However, as groups are possibly equivalent, by \Cref{prop:oracle_equivalent_optimal_loss_minimizers} we obtain that $f(\vp)$ may have additional weak maximizers in convex combinations of oracle-equivalent groups.
\end{proof}

\subsubsection{Heterogeneous fitness}
\label{subsec:appendix_retention_functions}

Here we extend \Cref{thm:optimal_learner} to settings where retention functions $\Retention_k$ vary across groups -- See \Cref{thm:optimal_learner_heterogeneous_retention} below. For $K=2$, we show that a modified version of the theorem holds for any pair of continuous, strictly increasing retention functions. For $K>2$, a similar conclusion follows under positive affine retention functions $\Retention_k(x)=a x - b_k$ with $a,b_k\ge0$. In each case, we prove the extended theorem by constructing a transformed potential function and showing that it retains the properties required for convergence and equilibrium structure guarantees.

\paragraph{Continuous increasing retention functions ($K=2$).} Consider a two-group setting, and denote $[K]=\Set{A,B}$. Let $\Learner^\opt$ be an oracle classification algorithm. Denote $\vp=(p,1-p)$, and
denote $h_p\sim\Learner^\opt\left((p,1-p)\right)$.

\begin{proposition}
\label{prop:two_action_optimal_classifier_monotone_fitness}
For any oracle algorithm, it holds that $\Loss_A(h_p)$ is decreasing in $p$, and $\Loss_B(h_p)$ is increasing in $p$.
\end{proposition}
\begin{proof}
Assume by contradiction that $\Loss_A(h_p)$ is not decreasing, and therefore there exist $p'>p$ such that $\Loss_A(h_p)< \Loss_A(h_{p'})$. 
The proof relies on a swapping argument. 
Denote:
\begin{align*}
a &= \Loss_A(h_p) \\
a' &= \Loss_A(h_{p'}) \\
b &= \Loss_B(h_p) \\
b' &= \Loss_B(h_{p'})
\end{align*}
By \Cref{claim:fixed_h_loss_convex_combination}, the average for loss for any fixed classifier is a convex combination of marginal loss functions. 
$h_p$ is optimal for $p$, and therefore:
$$
\underbrace{p a +(1-p) b}_{=\Loss_\vp(h_p)} \le \underbrace{p a'+ (1-p) b'}_{=\Loss_\vp(h_{p'})}
$$
Similarly, $h_{p'}$ is optimal for $p'$, and therefore:
$$
p' a'+(1-p')b' \le p' a+ (1-p') b
$$
Adding the two inequalities:
$$
(p'-p)(a' - a) \le (p'-p)(b' - b)
$$
Since $p'-p>0$, we divide by $(p'-p)$ to obtain:
$$
a' -a \le b' - b
$$
$a<a'$ by the contradiction assumption, and therefore by it holds that $b'-b>0$. But then:
$$
p' a +(1-p') b < p'a'+(1-p') b'
$$
In contradiction to the optimality of $h_{p'}$. Therefore, it must hold that $a\le a'$. Similarly, assume by contradiction that $b'>b$ and obtain a similar contradiction. \Cref{fig:optimal_classifier_potential} provides additional intuition.
\end{proof}

\begin{proposition}[Potential function of a two-group population game; {\citep[Example 3.2.3]{sandholm2010population}}]
Any two-group population game is a potential game, and the following is a potential function:
\begin{equation}
\label{eq:two_group_potential}
f((p,1-p))=\int_0^p\left(F_A((x,1-x))-F_B((x,1-x))\right)\mathrm d x
\end{equation}
\end{proposition}

\begin{proposition}
\label{prop:monotone_two_function_retention_unimodal_potential}
Let $\tilde f(\vp)$ be a potential function for a two-group evolutionary prediction game induced by an oracle classifier with increasing retention functions. Then $\tilde f(\vp)$ doesn't have local maximizers in the interior of the simplex.
\end{proposition}
\begin{proof}
Denote the transformed fitness function by $\tilde F_A((p,1-p))=\Retention_A\left(\acc_A(h_p)\right)$, and similarly denote $\tilde F_B((p,1-p))=\Retention_B\left(\acc_B(h_p)\right)$.
Let $g(p)=\tilde F_A((p,1-p))-\tilde F_B((p,1-p))$, and denote by $\tilde f((p,1-p))$ the potential function for the transformed game, defined by \Cref{eq:two_group_potential}.
By \Cref{prop:two_action_optimal_classifier_monotone_fitness}, it holds that $\tilde F_A((p,1-p))$ is increasing, $\tilde F_B((p,1-p))$ is decreasing, and therefore $g(p)$ is monotonically increasing as a function of $p$. By \Cref{eq:two_group_potential}, the potential $\tilde f$ is decreasing if $g(p)\le 0$ for all $p$, increasing if $g(p)\ge0$ for all $p$, and unimodal (decreasing and then increasing) if $g(p)$ changes signs. Note that $g(p)$ may change signs at most once due to monotonicity. In all of these cases, $g(p)$ doesn't have a local maximizer at the interior of the simplex.
\end{proof}

\paragraph{Positive affine retention functions with group-wise shift ($K\ge2$).}
Let $\Retention_k(x)=a x-b_k$ for $a,b_k\ge 0$. We show that the game admits a transformed potential function:
\begin{proposition}
Let $\Learner^\opt$ be an oracle classification algorithm, and let $f(\vp)$ be the potential function for the game $F_k(\vp)=\acc_k(h_\vp)$. Then the game $\tilde F_k(\vp)=\Retention_k\left(\acc_k(h_\vp)\right)=a \acc_k(h_\vp)-b_k$ admits the potential function:
\begin{align*}
\tilde f(\vp)
&=a f(\vp)-\vb\cdot\vp
\end{align*}
where $\vb=(b_1,\dots,b_K)\in\NonNegativeReals^K$.
\end{proposition}
\begin{proof}
First, consider the game $F^{(a)}_k=a F_k(\vp)$. By linearity of the gradient operation, the population game $F^{(a)}(\vp)$ is a potential game with potential $f^{(a)}(\vp)=af(\vp)$. Then, consider the game $F^{(b)}_k(\vp)=-b_k$. The game $F^{(b)}(\vp)$ is a constant game, and therefore admits the linear potential function $f^{(b)}(\vp)=-\vb\cdot\vp$ (see \citep[Proposition 3.2.13]{sandholm2010population}). 
Finally, by \citep[Section 3.2.4]{sandholm2010population}, the game $\tilde F(\vp)=F^{(a)}(\vp)+F^{(b)}(\vp)$ is a potential game with the potential function: $$\tilde f(\vp)=f^{(a)}(\vp)+f^{(b)}(\vp)=a f(\vp)-\vb\cdot\vp$$ as required.
\end{proof}

\begin{proposition}
\label{prop:postive_affine_retention_convex_potential}
Let $f(\vp)$ be a potential function for an evolutionary prediction game induced by an oracle classifier and positive affine retention functions with group-wise shift. Then $\tilde f(\vp)$ doesn't have local maximizers in the interior of the simplex.
\end{proposition}
\begin{proof}
By the proof of \Cref{thm:optimal_learner}, $F(\vp)$ is a potential game with potential function $f(\vp)=\acc_\vp(h_\vp)$.
By \Cref{claim:optimal_loss_is_concave} we obtain that $f(\vp)$ is convex over the simplex. Multiplying a convex function by a positive constant and adding a linear function maintains convexity, and therefore $\tilde f(\vp)$ is convex as well, and in particular does not have local maximizers in the interior of the simplex.
\end{proof}

\paragraph{\Cref{thm:optimal_learner} under heterogeneous retention.}
In both cases analyzed above, we identify a modified potential function which does not have local minimizers in the interior of the simplex (Propositions~\ref{prop:monotone_two_function_retention_unimodal_potential} and \ref{prop:postive_affine_retention_convex_potential}). We now proceed to state a modified version of the central theorem:
\begin{theorem}[Oracle classifier with heterogeneous retention; Extension of \Cref{thm:optimal_learner}]
\label{thm:optimal_learner_heterogeneous_retention}
Let $\Hypotheses$ be a hypothesis class, and denote by $\Learner^\opt$ the oracle learning algorithm with respect to $\Hypotheses$.
For $K=2$, assume that retention functions $\Retention_k$ are monotone increasing, and for $K>2$ assume that $\Retention_k$ are positive affine with group-wise shift.
Assume that
at each population state $\vp$
the oracle classifier $h_\vp\sim\Learner^\opt(D_\vp)$ is deployed and the fitness of group $k$ is $\tilde F_k(\vp)=\Retention_k\left(\acc_k(h_\vp)\right)$. 
Then it holds that:
\begin{enumeratecompact}
    \item \textbf{Welfare:} 
    Average population welfare increases over time:
    $
    \tfrac{\mathrm d}{\mathrm dt}
    \sum_k p_k\Retention_k\left(\acc_k(h_{\vp})\right) \ge 0
    $.
    \item \textbf{Stability:} A stable equilibrium always exists,
    and there can be multiple such equilibria.
    \item \textbf{Competitive exclusion:} 
    For all stable equilibria,
    $\Size{\support(\vp^*)}=1$.%
    \item \textbf{Coexistence:}
    Equilibria with $\Size{\support(\vp^*)}\ge2$ may exist,
    but are unstable.

\end{enumeratecompact}
\end{theorem}
\begin{proof}
In analogy to the proof of \Cref{thm:optimal_learner} (see \Cref{sec:retraining_proofs}). The corresponding evolutionary prediction games are potential games, with Nash equilibria corresponding to local equilibria of the potential function, and stable equilibria corresponding to local maximizers of the potential function.
By Propositions~\ref{prop:monotone_two_function_retention_unimodal_potential} and \ref{prop:postive_affine_retention_convex_potential}, potential functions for both families of retention functions don't have local maximizers at the interior of the simplex, implying claims (2,3,4).

For claim (1), we consider two cases, according to the family of retention functions under consideration.
For positive affine functions, the transformed potential function is
$
\tilde f(\vp)
= a f(\vp) - \vb\cdot\vp
$, where $f(\vp)$ is the potential function of the original game (with $\Retention_k(x)=x$). Rearranging the summations:
\begin{align*}
\tilde f(\vp)
&= a \cdot \underbrace{\acc_\vp(h_\vp)}_{=\sum_k p_k\acc_k(h_\vp)} - \vb\cdot\vp    
\\&= \sum_k p_k \left( a \cdot \acc_k(h_\vp) - b_k\right)
\\&= \sum_k p_k \Retention_k \left( \acc_k(h_\vp)\right)
\end{align*}
and claim (1) follows from the fact that $\tilde f(\vp)$ is a Lyapunov function. For monotone increasing functions, assume without loss of generality that $\tilde F_A(\vp)\ge \tilde F_B(\vp)$. Then from the positive correlation assumption $\dot p_A=-\dot p_B \ge 0$. From \Cref{prop:two_action_optimal_classifier_monotone_fitness} we obtain that $\acc_A(h_\vp)$ is increasing and $\acc_B(h_\vp)$ is decreasing as a function of $p_A$. Combining the observations, we obtain:
\begin{align*}
\tfrac{\mathrm d}{\mathrm dt} \sum_k p_k\Retention_k\left(\acc_k(h_{\vp})\right)
=& 
\tfrac{\mathrm d}{\mathrm dt} \left(
p_A\Retention_A\left(\acc_A(h_{\vp})\right)
+
p_B\Retention_B\left(\acc_B(h_{\vp})\right)
\right)
\\=&
\dot p_A 
\underbrace{
\left(
\tilde F_A(\vp)
-
\tilde F_B(\vp)
\right)
}_{\text{$\ge 0$ without loss of generality}}
\\&+ 
p_A 
\underbrace{\PartialDerivative{}{p_A}\Retention_A(\acc_A(h_\vp))}_{\text{$\ge 0$ by \Cref{prop:two_action_optimal_classifier_monotone_fitness}}}
\underbrace{\dot p_A}_{\ge 0}
\\&+
p_B 
\underbrace{\PartialDerivative{}{p_A}\Retention_B(\acc_B(h_\vp))}_{\text{$\le 0$ by \Cref{prop:two_action_optimal_classifier_monotone_fitness}}}
\underbrace{(-\dot p_A)}_{\le 0}
\end{align*}
And thus $\tfrac{\mathrm d}{\mathrm dt} \sum_k p_k\Retention_k\left(\acc_k(h_{\vp})\right) \ge 0$ as required.
\end{proof}

\subsection{Soft-SVM}
\label{subsec:soft_svm_coexistence_proof}
\begin{definition}[Soft-SVM; {\citep[e.g.,][Section 15.2]{shalev2014understanding}}]
\label{def:soft_svm}
Let $\Set{(x_i,y_i)}_{i=1}^n$ be a training set, and let $\lambda >0$. 
The Soft-SVM learning algorithm outputs a linear classifier $h(x)=\sign(w^*\cdot x+b)$ such that $(w^*, b^*)$ are minimizers of the regularized hinge loss:
$$
(w^*,b^*) = \argmin_{w,b} \lambda \Norm{w}^2 + \frac{1}{n} \sum_{i=1}^n \max\Set{0,1-y_i(wx_i+b)}
$$

\end{definition}

\subsubsection{Construction}
\label{subsec:soft_svm_construction}
Consider a two-group binary classification setting ($K=2, \Labels=\Set{-1,1}$). Denote the two groups by $[K]=\Set{A,B}$.
Denote by $\mathrm{Triangular}(a,b,c)$ the triangular distribution with parameters $(a,b,c)$ and probability density function: $$
\prob{x\sim\mathrm{Triangular}(a,b,c)}{x}=\begin{cases}
    \frac{2(x-a)}{(b-a)(c-a)}& x\in[a,c)\\
    \frac{2(b-x)}{(b-a)(b-c)}& x\in[c,b]\\
    0 & \otherwise
\end{cases}
$$
Let $\alpha\in[0.5,1)$ be a population balance parameter. 
Define the data distributions:
\begin{equation*}
\begin{aligned}
X_A&\sim \alpha \cdot \mathrm{Triangular}(-1,0,-1) + (1-\alpha)\cdot \mathrm{Triangular}(0,1,0) 
\\
X_B& = -X_A = (1-\alpha)\cdot \mathrm{Triangular}(-1,0,0) + \alpha\cdot\mathrm{Triangular}(0,1,1) 
\\
Y|X &= \begin{cases}
    1 & X\ge 0
    \\
    -1 & X < 0
\end{cases}
\end{aligned}
\end{equation*}
$D_A$ is the distribution of tuples $(X_A,Y|X_A)$, and $D_B$ is defined correspondingly as $D_B=(X_B,Y|X_B)$. The distributions are illustrated in \Cref{fig:coexistence_mechanisms} (Left).

\subsubsection{Stable Mutualistic Coexistence}
We show that the Soft-SVM algorithm can induce a stable, beneficial coexistence. 
We will consider learning in the population limit $n\to \infty$, where the loss minimization objective (\Cref{def:soft_svm}) converges towards its expected value:
\begin{definition}[Soft-SVM in the population limit]
Denote the probability density function of the data by $f(x,y)$. In the population limit $n\to\infty$, the Soft-SVM classifier $(w^*,b^*)$ is a minimizer of the expected regularized hinge loss:
\begin{equation*}
(w^*,b^*) = \argmin_{w,b} L(w,b)
\end{equation*}
Where:
\begin{equation}
\label{eq:expected_hinge_loss}
L(w,b) = \lambda \Norm{w}^2 + \sum_y \int_{x}f(x,y) \max\Set{0,1-y(wx+b)} \mathrm{d}x
\end{equation}
\end{definition}

\begin{proposition}
\label{lemma:1d_svm_b_derivative}
For a one-dimensional Soft-SVM classification problem, 
denote and feature-label distribution by $f$,
and denote by $L(w,b)$ the regularized hinge loss over the population. 
It holds that:
\begin{align}
\label{eq:1d_svm_b_derivative}
\frac{\partial L(w,b)}{\partial b}
&=
\int_{wx+b\ge -1} 
f(x,y=-1)
\mathrm{d}x
-
\int_{wx+b\le 1} 
f(x,y=1) 
\mathrm{d}x
\end{align}
\end{proposition}
\begin{proof}
The population hinge loss of a 1D classifier:
$$
L(w,b)=\lambda w^2 + \sum_y \int_{x\in\Reals} 
f(x,y) \max \Set{0,1-y(wx+b)}  
\mathrm{d}x
$$
First, we calculate the derivative $\frac{\partial L}{\partial b}$:
\begin{align*}
\frac{\partial L(w,b)}{\partial b}
&=
\frac{\partial}{\partial b}\sum_y \int_{x\in\Reals} 
f(x,y) \max \Set{0,1-y(wx+b)}  
\mathrm{d}x
\intertext{By the Leibniz integral rule, and excluding the single non-differential point from the integral:}
&=
\sum_y \int_{x\in\Reals\setminus\Set{ywx-1}} 
f(x,y) \frac{\partial}{\partial b} \max \Set{0,1-y(wx+b)}  
\mathrm{d}x
\\&=
\sum_y \int_{y(wx+b)\le 1} 
f(x,y) (-y)  
\mathrm{d}x
\intertext{Expanding the summation over $y$ yields:}
&=
\int_{wx+b\ge -1} 
f(x,y=-1)
\mathrm{d}x
-
\int_{wx+b\le 1} 
f(x,y=1) 
\mathrm{d}x
\end{align*}
As required. 

\end{proof}

\begin{proof}[Proof of \Cref{thm:soft_svm_coexistence}]
Assume the Soft-SVM algorithm trains on a large dataset sampled from $D_\vp$. 
Denote $\vp=(1-p,p)$, and denote the probability density functions of groups $\Set{A,B}$ by $f_A(x,y),f_B(x,y)$, respectively.
Since the data is separable and labels are positively correlated with the feature $x$, it holds that $w^*>0$. 
Assume that regularization parameter $\gamma$ is large enough such that $w^*\le 1$.
From \Cref{lemma:1d_svm_b_derivative}, it holds that:
$$
\left.
\frac{\partial L(w,b)}{\partial b}
\right|_{b=0}
=
\int_{wx\ge -1} 
f(x,y=-1)
\mathrm{d}x
-
\int_{wx\le 1} 
f(x,y=1)
\mathrm{d}x
$$
When $w\in(0,1]$, observe that the distribution defined in \Cref{subsec:soft_svm_construction} is suppported on $[-1,1]$, and therefore both integrals cover the whole distribution. We can therefore write:
$$
\left.
\frac{\partial L(w,b)}{\partial b}
\right|_{b=0}
=
\prob{(x,y)\sim D_\vp}{y=-1}
-
\prob{(x,y)\sim D_\vp}{y=1}
$$
By construction, for group $A$ it holds that 
$$
\prob{(x,y)\sim D_A}{y=-1}
-
\prob{(x,y)\sim D_A}{y=1}
=
\alpha - (1-\alpha) = -(1-2\alpha)
$$
and similarly for group $B$:
$$
\prob{(x,y)\sim D_B}{y=-1}
-
\prob{(x,y)\sim D_B}{y=1}
=
1-2\alpha
$$
Therefore, for $D_\vp$, it holds that:
$$
\left.
\frac{\partial L(w,b)}{\partial b}
\right|_{b=0}
=
p(1-2\alpha) - (1-p)(1-2\alpha)
$$

The state $\vp=(0.5,0.5)$ corresponds to $p=0.5$. For this state, it holds that
$\left.
\frac{\partial L(w,b)}{\partial b}
\right|_{b=0}=0$,
and therefore $b^*=0$.
Since the data is separable at $x=0$, any classifier with $w^*>0$ and $b^*=0$ achieves perfect accuracy, and therefore the state $\vp=(0.5,0.5)$ is beneficial for both groups.

To prove stability, we take an additional partial derivative by $p$:
$$
\left.
\frac{\partial^2 L(w,b)}{\partial p \partial b}
\right|_{b=0, p=0.5} 
=
2(1-2\alpha)
$$
And as $\alpha>0.5$, it holds that 
$
\left.
\frac{\partial^2 L(w,b)}{\partial p \partial b}
\right|_{b=0, p=0.5} 
<0
$.
Let $\varepsilon>0$ which is sufficiently small. For $p'=0.5+\varepsilon$ it holds that 
$\left.
\frac{\partial L(w,b)}{\partial b}
\right|_{b=0} 
<0$, and since $L(w,b)$ is convex, the optimal $b^*$ corresponding to $p'$ satisfies $b^*>0$, and the decision boundary is smaller than $0$. Assume that the deviation $\varepsilon$ is sufficiently small such that the new $b^*$ is also in the neighborhood of $0$. By definition, in a sufficiently small neighborhood of zero, it holds by construction for each group:
\begin{align*}
\acc_A(h_{\vp'})=\acc_A(w^*>0,b^*>0)
&=1-\int_{[-b^*/w^*,0]} f_A(x,y=-1) \mathrm{d}x
\\
\acc_B(h_{\vp'})=\acc_B(w^*>0,b^*>0)
&=1-\int_{[-b^*/w^*,0]} f_B(x,y=-1) \mathrm{d}x
\end{align*}
Hence, for a sufficiently small deviation $p'=p+\varepsilon$ it holds by the alignment of the triangle distributions that $\acc_A(h_{\vp'}) > \acc_B(h_{\vp'})$. Applying a similar argument to an opposite deviation $p''=0.5-\varepsilon$ shows that an opposite relation holds, and therefore the state $\vp=(0.5,0.5)$ is also stable.
\end{proof}

\subsubsection{Bifurcations Induced by Regularization}
\label{subsec:svm_regularization_experiment}

\begin{figure*}
    \centering
    \includegraphics[width=\textwidth]{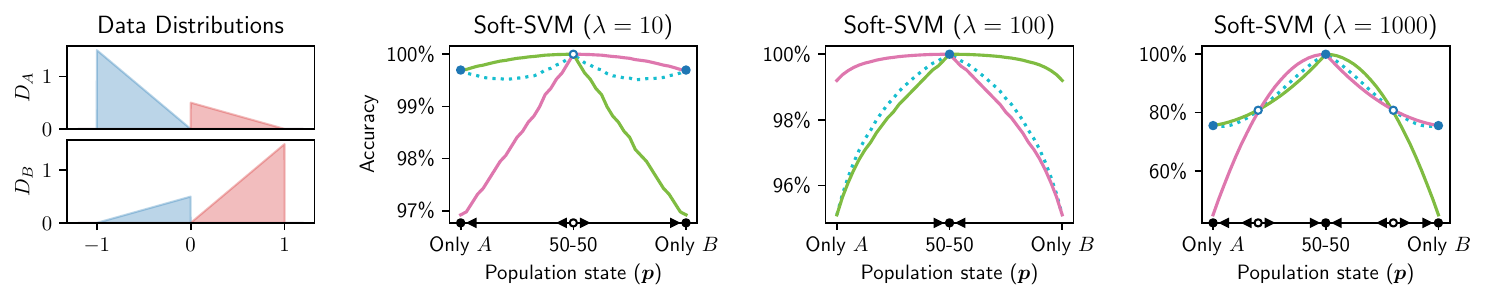}
    \vspace{-1em}
    \caption{
    Varying regularization strength for Soft-SVM classifiers.
    \textbf{(Left)} Data distributions.
    \textbf{(Center Left)}
    Evolutionary prediction game induced by a Soft-SVM classifier with $\lambda=10$. The game has two stable equilibria at the boundaries, and an unstable coexistence equilibrium.
    \textbf{(Center Right)}
    Game induced by a Soft-SVM classifier with $\lambda=100$. The game has one stable coexistence equilibrium.
    \textbf{(Right)}
    Game induced by a Soft-SVM classifier with $\lambda=1000$. The game has five equilibria: Two stable single-group equilibria, one stable beneficial coexistence equilibrium at $\vp=(0.5,0.5)$, and two unstable coexistence equilibria at $\vp\approx(0.2,0.8)$ and $\vp\approx(0.8,0.2)$.
    }
\label{fig:soft_svm_regularization}
\end{figure*}

In \Cref{fig:soft_svm_regularization}, we vary the regularization strength for the distributions specified in \Cref{subsec:soft_svm_construction} with imbalance parameter $\alpha=\ParamCoexistenceSvmRegularizationExperimentAlpha$. We observe that the system transition between three topologically-distinct phases: For weak regularization ($\lambda=10$), coexistence is beneficial but not stable. For intermediate regularization ($\lambda=100$), coexistence is beneficial and stable. Interestingly, for strong regularization ($\lambda=1000$), the game has five equilibria -- Two stable single-group equilibria, one stable coexistence equilibia, and two unstable coexistence equilibria. 

\subsection{k-NN Classifiers With Label Noise}
\label{subsec:knn_coexistence_proof}
\label{sec:coex-inductive_bias}

Another practical approach to loss minimization is 
to use a class of models that fit the training data perfectly and then interpolate.
Notable examples include the nearest-neighbor algorithm,
and over-parameterized neural networks \citep[e.g.,][]{zhang2021understanding}.
Interpolation leads to perfect accuracy on the training set by definition, but often at the price of sensitivity to noise;
as a concrete example, consider how \nnalg{} predictions change with the addition of a single point.
The problem is that even in the limit,
\nnalg{} cannot express the true $p(y|x)$,
and remains sensitive to label noise.
Our next result shows this can lead to coexistence:

\begin{theorem}
\label{thm:knn_coexistence}
For the \nnalg{} learning algorithm, there exist noisy-label data distributions that induce an evolutionary prediction game with stable coexistence.
\end{theorem}

Formal proof in \Cref{subsec:knn_coexistence_proofs}. The construction is illustrated in \Cref{fig:coexistence_mechanisms} (Center), and described formally in \Cref{subsec:knn_construction}.
Informally, there are two groups defined symmetrically. Each group is composed of a majority and minority classes, and the majority class has label noise. For each group alone, the \nnalg{} algorithm makes perfect predictions on its minority class, and imperfect predictions on the majority class.
Since groups are defined symmetrically,
training on one group alone leads to better accuracy for the other---leading to stable coexistence.
Moreover, \Cref{lemma:knn_expected_acc} shows that for $k$-NN with $k\ge3$,
stable coexistence can also be beneficial for both groups.

\begin{definition}[k-Nearest-Neighbors; {e.g. \citep[Section 19.1]{shalev2014understanding}}]
Let $k$ be an odd positive integer. Given a training set $S\in\left(\FeaturesLabels\right)^n$ and a feature vector $x$, denote the $k$ nearest neighbors of $x$ in the training set by $N(x)$. The k-NN classifier returns the majority label $y$ among the members of $N(x)$.
\end{definition}

\subsubsection{Construction}
\label{subsec:knn_construction}
Consider a two-group binary classification setting ($K=2$, $\Labels=\Set{-1,1}$).
Denote the two groups by $[K]=\Set{A,B}$, and let $\alpha>0$, $\beta\ge 0.5$.
The data distributions for the two groups are defined symmetrically:
\begin{itemizecompact}
    \item Each group is comprised of a majority subgroup and a minority subgroup. The proportion of the majority subgroup is $\beta\ge 0.5$.
    \item One subgroup is initially contains positively-labeled data, denoted by $D_\mathrm{pos}\in\DistOver{\Features\times\Set{1}}$.
    Conversely, the other subgroup is initially comprised of data with negative labels, and denoted by
    $D_\mathrm{neg}\in\DistOver{\Features\times\Set{-1}}$. The positive subgroup is the majority of group $A$, and the negative subgroup is majority of group $B$. 
    \item For each group, the labels in the majority subgroup are flipped with probability $\alpha$. For any data distribution $D\in\DistOver{\FeaturesLabels}$, we denote its noisy version by $\tilde{D}^\alpha$.
    \item It is assumed that the positive and negative data distribution have bounded support, and are sufficiently far apart such that the nearest neighbor of any sample is from the same group.
\end{itemizecompact}

Overall we have:
\begin{equation*}
\begin{aligned}
D_A &= \beta \tilde{D}^\alpha_\mathrm{pos} + (1-\beta)D_\mathrm{neg}
\\
D_B &= (1-\beta)D_\mathrm{pos} + \beta \tilde{D}^\alpha_\mathrm{neg}
\end{aligned}
\end{equation*}

The construction is illustrated in \Cref{fig:coexistence_mechanisms} (Center). 

\subsubsection{Stable Coexistence}
\label{subsec:knn_coexistence_proofs}

\begin{proposition}[Convex combinations]
\label{claim:label_noise_mixture_coefficients}
Let $\vp=(1-p,p)\in\DistOver{\Set{A,B}}$, and denote $\gamma=p -2p\beta +\beta$.
It holds that:
$$
D_\vp =
\gamma \tilde{D}^{\frac{\beta(1-p)}{\gamma}\alpha}_\mathrm{pos}
+
(1-\gamma) \tilde{D}^{\frac{\beta p}{1-\gamma}\alpha}_\mathrm{neg}
$$
\end{proposition}
\begin{proof}
For any mixture coefficient $q\in[0,1]$, the label-flipped datasets satisfy:
\begin{equation}
\label{eq:mixture_of_distributions_with_label_noise}
(1-q)\tilde{D}^\alpha + q D = \tilde{D}^{(1-q)\alpha}
\end{equation}
The convex combination of the mixture coefficients of corresponding to $D_\mathrm{pos}$ and $\tilde{D}^\alpha_\mathrm{pos}$, is given by:
\begin{equation}
\label{eq:mixture_coefficients_for_label_noise}
(1-p)\beta + p(1-\beta)=p+\beta-2p\beta
=\gamma
\end{equation}
Plugging \cref{eq:mixture_coefficients_for_label_noise} into \cref{eq:mixture_of_distributions_with_label_noise}, we obtain that the corresponding $q$ for the positive mixture is:
$$
q_\mathrm{pos}=\frac{(1-p)\beta}{(1-p)\beta + p(1-\beta)}=\frac{\beta(1-p)}{\gamma}
$$
and for the negative mixture:
$$
q_\mathrm{neg}=
\frac{p \beta}{p\beta + (1-p)(1-\beta)}
=\frac{\beta p}{1-\gamma}
$$
Combining the two equations above yields the result.
\end{proof}

\begin{proposition}[Reflection-exchange symmetry]
\label{claim:label_noise_symmetry_argument}
Let $\vp=(1-p,p)$, and let $\vp'=(p,1-p)$. For the prediction setting defined above and for any classifier $h^*$ trained on a dataset $S\sim D^n_\vp$, it holds that:
$$
\expect{S\sim D_{\vp}}{\acc_{D_A}(h^*)}
=\expect{S\sim D_{\vp'}}{\acc_{D_B}(h^*)}
$$
\end{proposition}
\begin{proof}
From symmetry of the definition. 
Intuitively, $\expect{S\sim D_{\vp}}{\acc_{D_B}(h^*)}$ can be obtained by reflecting $\expect{S\sim D_{\vp}}{\acc_{D_A}(h^*)}$ across the $\vp$ axis.
\end{proof}

\begin{proposition}[Expected prediction of $k$-NN]
\label{claim:knn_expected_prediction}
Let $D_x\in\DistOver{\Features}$ be a distribution over features, let $\alpha>0$ be a noise parameter, and let $k$ be an odd number greater or equal to $1$. Denote the label distribution by $D^\alpha_y=1-2\mathrm{Bernoulli}(\alpha)$, and denote the joint feature-label distribution by $D=D_x\otimes D_y^\alpha \in \DistOver{\FeaturesLabels}$. Let $h:\Features\to\Set{-1,1}$ be a $k$-Nearest-Neighbors ($k$-NN) classifier trained on a dataset $S\sim D^n$ with $n\ge k$, and let $x\in \Features$.
It holds that:
$$
\expect{S\sim D^n}{h(x)} = -1+2\phi_k(\alpha)
$$
Where $\phi_k(\alpha)$ the cumulative distribution function (CDF) of a $\mathrm{Binomial}\left(k, \alpha\right)$ random variable, taken at $\left\lfloor \frac{k}{2} \right\rfloor$:
$$
\phi_k(\alpha) = \Pr\left(\mathrm{Binomial}\left(k, \alpha\right) \le \left\lfloor \frac{k}{2} \right\rfloor\right)
$$
\end{proposition}
\begin{proof}
Denote the $k$-NN training set by $S\sim D^n$. For any $x\in\Features$, denote its $k$ nearest neighbors by $N(x) \in \left(\FeaturesLabels\right)^n$. Denote by $N_\mathrm{neg}(x)$ the number of neighbors with negative labels:
$$
N_\mathrm{neg}(x) = \Size{\Set{(x',y')\in N(x)\mid y'=-1}}
$$
Since features and labels in $D$ are assumed to be independent, the number of neighbors with negative labels is a binomial random variable:
$$
N_\mathrm{neg}(x) \sim \mathrm{Binomial}\left(
k,\alpha\right)
$$
As a $k$-NN classifier predicts the label according to the majority label in $N(x)$, the expectation value of the label of $x$, taken over the randomness of the training set, depends on the cumulative distribution function of $N_\mathrm{neg}(x)$:
\begin{equation*}
\begin{aligned}
\expect{S\sim D^n}{h(x)}
&=-1+2\Pr\left(h(x)=1\right)
\\&=
-1+2\Pr\left(N_\mathrm{neg}(x)\le \left\lfloor\frac{k}{2}\right\rfloor\right)
\end{aligned}
\end{equation*}
And using the definition of $\phi_k(\alpha)$, we obtain:
$$
\expect{S\sim D^n}{h(x)}
= -1+2\phi_k(\alpha)
$$
\end{proof}

\begin{lemma}[Expected accuracy]
\label{lemma:knn_expected_acc}
For the distributions $D_A$, $D_B$ defined above, let $\vp=(1-p,p)\in\DistOver{\Set{A,B}}$, and let $k>0$ be an odd integer. Denote by $h_\vp$ the $k$-Nearest-Neighbor ($k$-NN) classifier trained on data sampled from $D_\vp$. Assume that the training set contains feature vectors from both subgroups, and that the supports of the feature vector distributions in the datasets $D_\mathrm{pos}$, $D_\mathrm{neg}$ are sufficiently far apart. 
The expected accuracy of $h_\vp$ with respect to $D_A$ is:
\begin{equation}
\label{eq:knn_expected_acc_a}
\begin{aligned}
\expect{S\sim D_\vp^n}{\acc_{D_A} (h^*)} =&
\beta\left(
\frac{1+(1-2\alpha)
\left(-1+2\phi_k\left(\frac{\beta(1-p)}{\gamma}\alpha\right)\right)
}{2}
\right)
\\&+
(1-\beta)\left(
\frac{1+(-1)
\left(1-2\phi_k\left(\frac{\beta p}{1-\gamma}\alpha\right)\right)
}{2}
\right)
\end{aligned}
\end{equation}
Where $\phi_k(\alpha)$ the cumulative distribution function of a $\mathrm{Binomial}\left(k, \alpha\right)$ variable, taken at $\left\lfloor \frac{k}{2} \right\rfloor$.
\end{lemma}
\begin{proof}

Denote the training set by $S=\Set{(x_i,y_i)}_{i=1}^n \sim D_\vp^n$, and denote the learned 1-NN classifier by $h_\vp$. 
For any $x\in\Features$, denote its nearest neighbor in the training set by $i^*(x)=\argmin_{i\in[n]} d(x,x_i)$.
The expected accuracy of $h_\vp$ with respect to $D_A$ is given by:
\begin{equation*}
\begin{aligned}
\expect{S\sim D_\vp^n}{\acc_{D_A}(h)}
&= 
\expect{S\sim D_\vp^n, (x,y)\sim D_A}{\frac{1+y y_{i^*(x)}}{2}}
\end{aligned}
\end{equation*}
Since the supports of the subgroups are sufficiently far apart, the nearest neighbor $i^*(x)$ always originates from the same subgroup as $x$.
From this we obtain that $y$ and $y_{i^*(x)}$ are independent given the subgroup from which $(x,y)$ was sampled ($\in\Set{\mathrm{pos}, \mathrm{neg}}$). Thus, for noisy data from the positive subgroup ($D^\alpha_\mathrm{pos}$), we have:
\begin{equation*}
\begin{aligned}
\expect{S\sim D_\vp^n}{\acc_{\tilde{D}_\mathrm{pos}^\alpha}(h)}
&=
\expect{S\sim D_\vp^n, (x,y)\sim \tilde{D}^\alpha_\mathrm{pos}}{\frac{1+y y_{i^*(x)}}{2}}
\\&= 
\frac{1+
\expect{
}{y}
\expect{
}{y_{i^*(x)}}
}{2}
\\&=
\frac{1+(1-2\alpha)
\left(-1+2\phi_k\left(\frac{\beta(1-p)}{\gamma}\alpha\right)\right)
}{2}
\end{aligned}
\end{equation*}
Where the identity 
$\expect{}{y}=(1-2\alpha)$ is given by the definition of $\tilde{D}_\mathrm{pos}^\alpha$, the
definition of $\gamma$ is given by \Cref{claim:label_noise_mixture_coefficients}, 
and the identity $\expect{}{y_{i^*(x)}}=
\left(-1+2\phi_k\left(\frac{\beta(1-p)}{\gamma}\alpha\right)\right)
$ 
is given by \Cref{claim:knn_expected_prediction}.

Similarly, for clean data from the negative subgroup ($D_\mathrm{neg}$), we obtain:
\begin{equation*}
\begin{aligned}
\expect{S\sim D_\vp^n}{\acc_{D_\mathrm{neg}}(h)}
&=
\expect{S\sim D_\vp^n, (x,y)\sim D_\mathrm{neg}}{\frac{1+y y_{i^*(x)}}{2}}
\\&= 
\frac{1+
\expect{
}{y}
\expect{
}{y_{i^*(x)}}
}{2}
\\&=
\frac{1+(-1)
\left(1-2\phi_k\left(\frac{\beta p}{1-\gamma}\alpha\right)\right)
}{2}
\end{aligned}
\end{equation*}
and jointly:
\begin{equation*}
\begin{aligned}
\expect{S\sim D_\vp^n}{\acc_A(h)}
=& 
\beta
\expect{S\sim D_\vp^n}{\acc_{\tilde{D}_\mathrm{pos}^\alpha}(h)}
+
(1-\beta)
\expect{S\sim D_\vp^n}{\acc_{D_\mathrm{neg}}(h)}
\end{aligned}
\end{equation*}
And finally, plugging in the expected accuracy values calculated above yields the desired result.
\end{proof}

\begin{figure*}
    \centering
    \includegraphics[width=\textwidth]{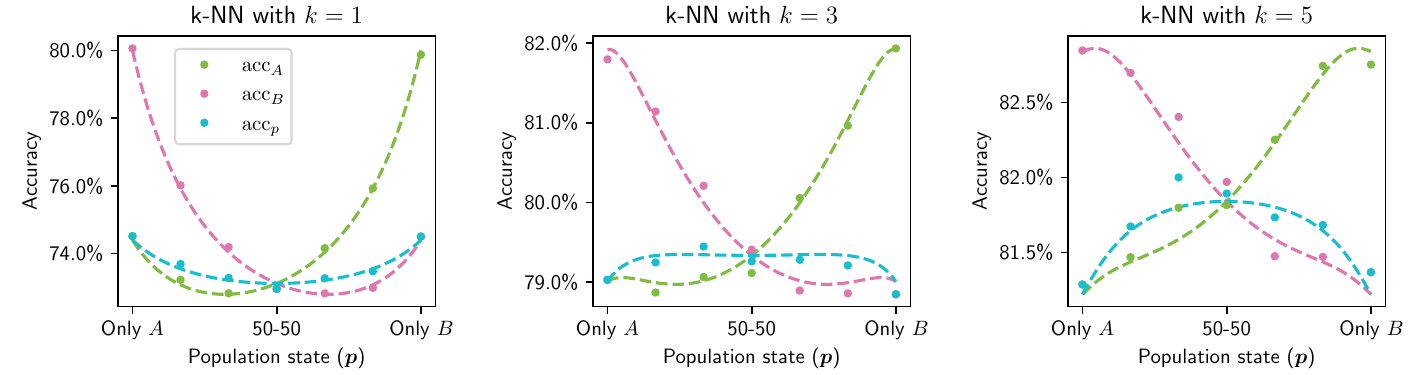}
    \vspace{-1em}
    \caption{
    k-NN accuracy for the construction specified in \Cref{subsec:knn_construction}. The plots show alignment between the theoretical curves
    (\Cref{lemma:knn_expected_acc}) and empirical simulations (mean accuracies represented by dots). For $k=3,5$ the game has a stable beneficial coexistence.
    }
\label{fig:knn_theory}
\end{figure*}

\begin{proposition}
\label{prop:1nn_coexistence}
For $k=1$, a prediction game induced by a $1$-NN classifier trained on dataset with label noise has the following properties:
\begin{enumerate}
    \item A stable coexistence equilibrium exists if and only if $\alpha \in \left(0,1-\frac{1}{2\beta}\right)$. If a stable coexistence equilibrium exists, then it exists at the uniform population state $\vp=(0.5,0.5)$.
    \item The overall welfare at the coexistence state $\vp=(0.5,0.5)$ is lower than the overall welfare at extinction states $\vp\in\Set{(1,0),(0,1)}$.
\end{enumerate}
\end{proposition}
\begin{proof}
By applying \Cref{lemma:knn_expected_acc}. When $k=1$, the $\phi_k$ function satisfies:
$$
\phi_{k=1}(\alpha)=\Pr\left(\mathrm{Binomial}(1,\alpha)\le 0\right)=1-\alpha
$$
Plugging into \cref{eq:knn_expected_acc_a} and using the symmetry argument in \Cref{claim:label_noise_symmetry_argument}, we obtain:
\begin{equation*}
\begin{aligned}
\acc_A\left(h^*_{\vp=(1,0)}\right)
=
\acc_B\left(h^*_{\vp=(0,1)}\right)
&= 2\alpha\beta(\alpha-1) + 1
\\
\acc_B\left(h^*_{\vp=(1,0)}\right)
=
\acc_A\left(h^*_{\vp=(0,1)}\right)
&= 1-\alpha
\\
\acc_A\left(h^*_{\vp=(0.5,0.5)}\right)
=
\acc_B\left(h^*_{\vp=(0.5,0.5)}\right)
&= 2\alpha\beta(\alpha\beta-1)+1
\end{aligned}
\end{equation*}
For (1), a stable equilibrium exists if:
$$
\acc_A\left(h^*_{\vp=(1,0)}\right)<\acc_B\left(h^*_{\vp=(1,0)}\right)
$$
Plugging in the values calculated above:
$$
2\alpha\beta(\alpha-1)+1<1-\alpha
$$
Which is satisfied when $\alpha\in\left( 0, 1-\frac{1}{2\beta} \right)$.

For (2), we note that $\beta\le1$, and therefore:
$$
\acc_A\left(h^*_{\vp=(1,0)}\right) - \acc_A\left(h^*_{\vp=(0.5,0.5)}\right)
=
\alpha\beta(\alpha-\alpha\beta)
=
\alpha^2\beta(1-\beta)
\ge 0
$$
\end{proof}

\begin{proof}[Proof of \Cref{thm:knn_coexistence}]
By \Cref{prop:1nn_coexistence}, using e.g. $(\alpha,\beta)=(0.2,0.8)$. We observe that the conditions of the proposition are satisfied for this choice of constants, as $\alpha=0.2<0.375=1-\frac{1}{2\cdot0.8}=1-\frac{1}{2\beta}$.
\end{proof}

\subsection{Hard-SVM With Finite Data}
\label{subsec:finite_data_coexistence_proof}

All practical learning algorithms must rely on a finite sample of data points for training a classifier.
Learnability (in the PAC sense) implies that as data size grows,
the sample becomes increasingly representative of the distribution.
But the \emph{rate} at which variance diminishes need not be equally across all regions of a distribution.
In a sense, this means that one region can be learned `faster' (i.e., with less data) than another.
Our final example exploits this to enable coexistence.

\begin{theorem}
\label{thm:finite_data_coexistence}
For the Hard-SVM algorithm, 
there exist linearly-separable data distributions
induce an evolutionary prediction game with
beneficial coexistence.
\end{theorem}

Proof in \Cref{subsec:hard_svm_coexistence_proof}. The construction is illustrated in \Cref{fig:coexistence_mechanisms} (Right), and described formally in \Cref{subsec:hard_svm_construction}. Informally, the key property we leverage here is asymmetric class variance. In each group, one class has high variance, and the other class has low variance. When the dataset is of size $n$, this creates bias against the high-variance class, and thus convergence towards the optimal classifier at rate $O(n^{-1})$.
When the two groups coexist, this bias is balanced, and the rate of convergence is exponential---leading to coexistence
that is beneficial, but unstable.
\squeeze

\begin{definition}[Hard-SVM; {e.g. \citep[Section 15.1]{shalev2014understanding}}]
Given a linearly-separable training set $\Set{(x_i,y_i)}_{i=1}^n \in \left(\Reals^d\times \Set{-1,1}\right)^n$, the Hard-SVM algorithm outputs a linear classifier $\sgn(w^*\cdot x+b^*)$ which maximizes the margin of the decision hyperplane:
$$
(w^*, b^*) = \argmax_{(w,b); \Norm{w}=1} \min_{i\in[n]} y_i(w\cdot x+b)
$$
\end{definition}
\begin{remark}[One-dimensional Hard-SVM]
\label{remark:1d_hard_svm}
For one dimensional data, 
denote by $x^+_\mathrm{min}$ the minimal $x$ with a positive label in the training set by $x^+_\mathrm{min}=\min_i \Set{x_i \mid y_i=1}$. Similarly, denote by $x^-_\mathrm{max}$ the maximal $x$ with a negative label in the training set by $x^-_\mathrm{max}=\max_i \Set{x_i \mid y_i=-1}$.
When $x^+_\mathrm{min}\ge x^-_\mathrm{max}$, the Hard-SVM decision margin is given by: 
\begin{equation}
\label{eq:1d_hard_svm}
x_\mathrm{margin}
=
\frac{
x^-_\mathrm{max}
+
x^+_\mathrm{min}
}{2}
\end{equation}
\end{remark}
\subsubsection{Construction}
\label{subsec:hard_svm_construction}
Let $\varepsilon>0$, and denote by $\delta(x=a)$ the Dirac delta distribution centered around $x=a$.

\begin{equation*}
\begin{aligned}
X_A&\sim 0.5 \left(\mathrm{Uniform}\left([0,1]\right)+\varepsilon\right) + 0.5\delta(x=-\varepsilon)
\\
X_B& = -X_A
\\
Y|X &= \begin{cases}
    1 & X\ge 0
    \\
    -1 & X < 0
\end{cases}
\end{aligned}
\end{equation*}

$D_A$ is the distribution of tuples $(X_A,Y|X_A)$, and $D_B$ is defined correspondingly as $D_B=(X_B,Y|X_B)$. The distributions are illustrated in \Cref{fig:coexistence_mechanisms} (Right).

\subsubsection{Mutualistic Coexistence}
\label{subsec:hard_svm_coexistence_proof}
\begin{proposition}
\label{claim:finite_sample_xmin_lower_bound}
For a given positive integer $N$, Let $n\sim \mathrm{Binomial}\left(N,\frac{1}{2}\right)$, let $x_1,\dots,x_n\sim\mathrm{Uniform}\left([0,1]\right)$, and let $x_\mathrm{min}=\min_{i\in[n]} x_i$. Then for $N\ge 15$, it holds that:
$$
\expect{}{x_{\mathrm{min}}} \ge \frac{1}{N}
$$
\end{proposition}
\begin{proof}
Let $n\sim \mathrm{Binomial}\left(N,\frac{1}{2}\right)$. 
By Hoeffding's inequality,
it holds that:
$$
\prob{n}{n-\frac{N}{2} \ge t}\le e^{-\frac{2t^2}{N^2}}
$$
And for $t=\frac{N}{4}$ we obtain:
\begin{equation}
\label{eq:finite_data_k_hoeffding}
\prob{n}{n \ge \frac{3N}{4}} \le e^{-\frac{N}{8}}
\end{equation}
For any fixed $n$, let
$x_1,\dots,x_n\sim\mathrm{Uniform}\left([0,1]\right)$, and
denote the first order statistic among $\Set{x_i}$ by $x_{\mathrm{min},n}=\min_{i\in[n]} x_i$. As each $x_i$ is a uniform random variable over the unit interval, the first order statistic admits a beta distribution:
$$
x_{\mathrm{min},n}\sim \mathrm{Beta}(1,n)
$$
and therefore:
\begin{equation}
\label{eq:finite_data_x_min_k_expectation}
\expect{}{x_{\mathrm{min},n}} = \frac{1}{n+1}
\end{equation}
Now consider the compound variable $x_{\mathrm{min}}=x_{\mathrm{min,n}}$ where $n\sim\mathrm{Binomial}\left(N,\frac{1}{2}\right)$. Applying the law of total expectation:
\begin{align*}
\expect{}{x_\mathrm{min}}
&=
\expect{n}{\expect{}{x_\mathrm{min}\mid n}}
\\&=
\expect{n}{\expect{}{x_\mathrm{min,n}}}
\intertext{By \cref{eq:finite_data_x_min_k_expectation} we obtain the expectation of $x_{\mathrm{min},n}$:}
\expect{}{x_\mathrm{min}}&=
\expect{n}{\frac{1}{n+1}}
\intertext{\cref{eq:finite_data_k_hoeffding} gives a lower bound on the expectation:}
\expect{}{x_\mathrm{min}}
&\ge 
\overbrace{
\frac{1}{\frac{3N}{4}+1}\cdot
\left(1-e^{-\frac{N}{8}}\right)
}^{
n < \frac{3N}{4}
}
+ 
\overbrace{
0\cdot e^{-\frac{N}{8}}
}^{n\ge\frac{3N}{4}}
\intertext{Then for $N\ge 15$ it holds that:}
\expect{}{x_\mathrm{min}}&\ge \frac{1}{N}
\end{align*}
As required.

\end{proof}

\begin{proposition}
\label{prop:finite_data_coexistence_formal}
For the distribution defined above, assume the Hard-SVM classifier is trained on dataset of size $n\ge 21$, and let $\varepsilon\in\left(0,\frac{1}{2n}\right)$. It holds that:
\begin{enumerate}
    \item When the classifier is trained on $D_\vp$ for $\vp=(1,0)$:
    \begin{align*}
    \expect{}{\acc_\vp(h)}
    =\expect{}{\acc_A(h)}
    &\le 1-\frac{1}{8n}
    \end{align*}
    \item When the classifier is trained on $D_\vp$ for $\vp=(0,1)$:
    \begin{align*}
    \expect{}{\acc_\vp(h)}
    =\expect{}{\acc_B(h)}
    &\le 1-\frac{1}{8n}
    \end{align*}
    \item When the classifier is trained on $D_\vp$ for $\vp=(0.5,0.5)$:
    \begin{align*}
    \acc_{\vp}(h_n) =
    \expect{}{\acc_A(h)} = \expect{}{\acc_B(h)}
     \ge 1-2\left(1-\frac{1}{4}\right)^n
    \end{align*}
    \item Coexistence is beneficial.
\end{enumerate}
\end{proposition}
\begin{proof}
For case (1), the Hard-SVM algorithm trains on data sampled from $D_A$.
Denote the size of the dataset by $n$ and the corresponding Hard-SVM classfier by $h_n$. We note that $h_n$ is a random variable depending on the dataset. 
For a dataset $\Set{(x_i,y_i)}_{i=1}^n\sim D_A^n$,
Since the problem is one-dimensional, the Hard-SVM decision margin is given by \cref{eq:1d_hard_svm}: 
$$
x_\mathrm{margin}
=
\frac{
x^-_\mathrm{max}
+
x^+_\mathrm{min}
}{2}
=\frac{x^+_\mathrm{min}-\varepsilon}{2}
>0
$$
And the expected accuracy of each group is given by:
\begin{align*}
\expect{}{\acc_A(h_n) \mid \omega^c}
&= \frac{1}{2}+\frac{1}{2}\left(1-x_\mathrm{margin} + \varepsilon\right)
= 1-\frac{x^+_\mathrm{min}-\varepsilon}{4}
+ \frac{\varepsilon}{4}
\end{align*}

By \Cref{claim:finite_sample_xmin_lower_bound} and for $n\ge 15$, it holds that $\expect{}{x_\mathrm{margin}}\ge\frac{1}{n}+\varepsilon$, and therefore:
\begin{align*}
\expect{}{\acc_A(h_n) \mid \omega^c}
&\le 1- \frac{1}{4n}+\frac{1}{8n}=1-\frac{1}{8n}
\end{align*}

Case (2) follows from symmetry.

For case (3), denote by $\omega$ the bad event in which the training set $\Set{(x_i,y_i)}_{i=1}^n\sim D_\vp^n$ does not contain $x_i\in\Set{-\varepsilon,\varepsilon}$. By the union bound, the probability of this event is bounded by $2\left(1-\frac{1}{4}\right)^n$, and therefore the accuracy for $\vp=(0.5,0.5)$ is bounded from below by:
$$
\acc_{\vp}(h_n) \ge 1-2\left(1-\frac{1}{4}\right)^n
$$

Finally, for (4), we note that 
$1-2\left(1-\frac{1}{4}\right)^n \ge 1-\frac{1}{8n}$ for all $n\ge 21$.
\end{proof}

\begin{proof}[Proof of \Cref{thm:finite_data_coexistence}]
By \Cref{prop:finite_data_coexistence_formal}.
\end{proof}

\subsection{Stabilizing Coexistence}
\label{subsec:steering_coexistence_proof}

\begin{proof}[Proof of \Cref{prop:stabilizing_coexistence}]
Let $F(\vp)$ be an evolutionary prediction game induced by an optimal learning algorithm $\Learner(\vp)$, and let $\vp^*$ be an unstable coexistence equilibrium with full support ($\Size{\support(\vp^*)}=K$). 
The game induced by $\Learner'(\vp)$ in the neighborhood of $\vp^*$ is $F'(\vp)=F(2\vp^*-\vp)$.
By \Cref{lemma:optimal_predictor_game_is_potential}, the game $F(\vp)$ is a potential game, and therefore admits a potential function $f(\vp)=\acc_\vp(h_\vp)$. Consider the potential function $f'(\vp)=-f(2\vp^*-\vp)$. By the chain rule, it holds that:
$$
\nabla f'(\vp) = -(\nabla f )(2\vp^*-\vp)\cdot\nabla(2\vp^*-\vp) = F(2\vp^*-\vp) = F'(\vp)
$$
And therefore $f'(\vp)$ is a potential function for the game $F'(\vp)$ in the neighborhood of $\vp^*$. 
The equilibrium $\vp^*$ is unstable, and therefore by \citep[Theorem 8.2.1]{sandholm2010population} it is not a maximizer of $f(\vp)$. Moreover, by \Cref{claim:optimal_loss_is_concave} the potential function $f(\vp)=\acc_\vp(h_\vp)$ is convex, and as $\vp^*$ is assumed to have full support, it is also a global minimizer of $f(\vp)$. From this we conclude that $\vp^*$ is a global maximizer of $f'(\vp)$, and therefore the equilibrium $\vp^*$ is stable under the evolutionary game induced by $\Learner'$.
\end{proof}

\section{Implementation Details}
\label{sec:appendix_experiments}

\paragraph{Code.} We implement our simulations and analysis in Python. Our synthetic-data experiments rely on scikit-learn \citep{scikit-learn} for learning algorithm implementations, our CIFAR-10 and MNIST experiments rely on PyTorch \citep{paszke2019pytorch} and \texttt{ffcv} \citep{leclerc2023ffcv}, and our ACSIncome experiment relies on scikit-learn and XGBoost \citep{chen2016xgboost}. We  use \texttt{matplotlib} \citep{matplotlib} for plotting, and \texttt{mpltern} \citep{yuji_ikeda_2024_11068993} for ternary plots.
Code is available at:
\url{https://github.com/edensaig/evolutionary-prediction-games}.

\paragraph{Hardware.} Synthetic data simulations were run on a single Macbook Pro laptop, with 16GB of RAM, M2 processor, and no GPU. Experiments involving neural networks (\Cref{sec:empirical}) were run on a dedicated server with an AMD EPYC 7502 CPU, 503GB of RAM, and an Nvidia RTX A4000 GPU.

\paragraph{Runtime.} A single run of the complete synthetic data pipeline takes roughly 30 minutes on a laptop. A single repetition of each real-data experiment (i.e. sampling $h\sim\Learner(\vp)$ and computing marginal accuracies) takes roughly 10 minutes.

\paragraph{Architectures.} For CIFAR-10, we use the Resnet-9 architecture and training code provided in the CIFAR-10 example code in the \texttt{ffcv} Github repository (\texttt{libffcv/ffcv}, commit \texttt{7885f40}), with modifications to control the probability of horizontal flips in training and testing. 
For MNIST, we use the convolutional neural network provided in the MNIST example code in the PyTorch examples Github repository (\texttt{pytorch/examples}, commit \texttt{37a1866}). The MNIST network has two convolutional layers, and two fully-connected layers, with dropout and max pooling. Training is performed for 200 epochs using SGD with learning rate $0.01$ and momentum $0.5$. For the ACSIncome experiments, we use scikit-learn and XGBoost classifiers with default regularization parameters (\texttt{LinearSVC}, \texttt{LogisticRegression}, \texttt{XGBClassifier}).

\paragraph{Replicator dynamics simulation.} 
In the experiments, we use the discrete replicator equation for simulating evolutionary dynamics. The continuous-time replicator equation is:
$$\dot p_k=p_k\left(F_k(\vp)-\bar F(\vp)\right)$$
where $\bar F(\vp)=\sum_k p_k F_k(\vp)$ is the average fitness across the population. In our simulations  (e.g., Figures~\ref{fig:optimal_classifier_retraining}, \ref{fig:empirical}), we use the discrete-time replicator equation given by \citet[eq. (3)]{taylor1978evolutionary}:
\begin{equation}
\label{eq:discrete_replicator}
p^{t+1}_k = p^t_k \frac{F_k(\vp)+1}{\bar F(\vp)+1}
\end{equation}
For the CIFAR-10 replicator simulation (\Cref{subsec:cifar_experiment}), we discretize the two-group simplex into the grid $\Set{(1,0),(0.9,0.1),\dots,(0,1)}$, precompute fitness values $F_k(\vp)$ at each grid point, and use linear interpolation to determine fitness at intermediate states.

\paragraph{Confidence intervals.} 
In both experiments, variability is due to the random split within the train and test sets, and due to random noise when it is added.
For numerical results, we report mean confidence intervals at $99\%$ confidence level.

\section{Additional Empirical Evaluation}
\label{app:additional_experiments}

\subsection{Effect of Sampling Noise} 
\label{app:sampling_noise_sensitivity_analysis}

\begin{figure*}
    \centering
    \includegraphics[width=\textwidth]{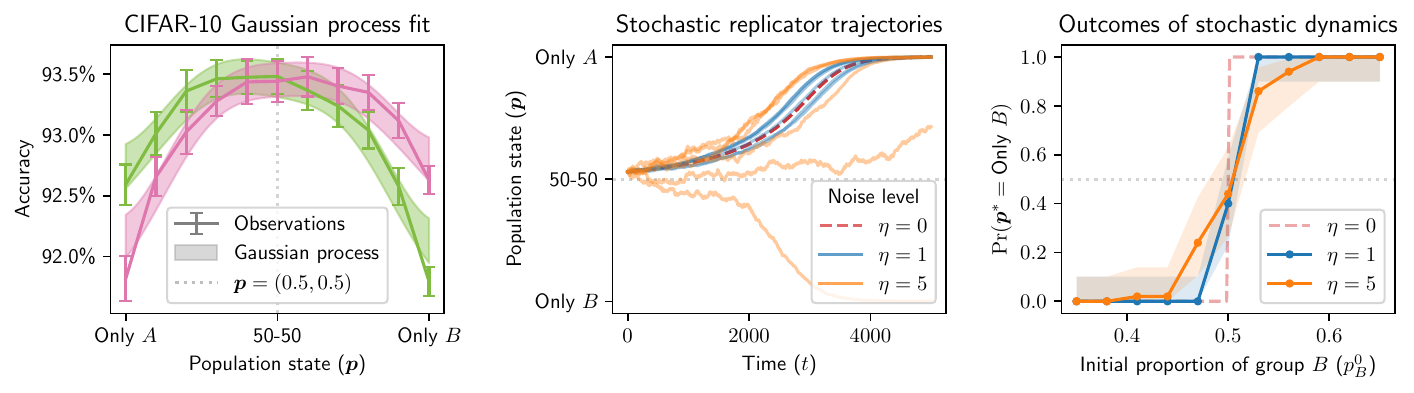}
    \vspace{-1em}
    \caption{
    Effect of sampling noise in the CIFAR-10 setting (\Cref{app:sampling_noise_sensitivity_analysis}). 
    \textbf{(Left)}
    Gaussian process fit to the raw observations. Error bars indicate standard deviation of data, and shaded areas indicate standard deviation of the Gaussian process.
    \textbf{(Center)}
    Trajectory samples of the induced stochastic replicator dynamics, for different simulated noise levels $\alpha$.
    \textbf{(Right)}
    Distribution of long-term outcomes as a function of initial condition.
    }
\label{fig:empirical_sensitivity_stochastic}
\end{figure*}
The prediction games defined in \Cref{sec:prediction_games}
associate evolutionary fitness with accuracy in expectation over the training set (\cref{eq:population_game_fitness}, a reasonable assumption where sampling noise is small. Here we test the robustness of this assumption by quantifying the effect of sampling noise on long-term outcomes. 

\paragraph{Method.} We assess the effect of sampling noise on long term outcome by fitting a Gaussian process to the outcomes of different sampled training sets, and analyzing the resulting stochastic dynamics for varying initial conditions. We use the raw data from the CIFAR-10 experiment  introduced in \Cref{subsec:cifar_experiment} (group accuracies for each training set sample).
We fit a Gaussian process (GP) with an RBF kernel with length scale $\ParamSensitivityGaussianProcessLengthScale$, using the skicit\nobreakdashes-learn implementation with alpha regularization parameter $\ParamSensitivityGaussianProcessAlpha$. 
To generate stochastic replicator dynamics, we replace the deterministic fitness terms in \cref{eq:discrete_replicator} with samples from the Gaussian process. We simulate different noise levels by multiplying the Gaussian process covariance by a constant $\eta \ge 0$, such that $\eta=0$ coincides with the deterministic dynamics, $\eta=1$ coincides with the Gaussian process fit to the original data, and $\eta=5$ represents a simulated five-fold increase in compared to the observed sampling noise. 
For each $\eta$, we simulate stochastic dynamics for a range of initial states  $p^0_B\in[\ParamSensitivityInitialStateRangeMin, \ParamSensitivityInitialStateRangeMax]$, and measure the distribution of long term outcomes (dominance of group $A$ vs. group $B$) over repeated realizations.

\paragraph{Results.}
Gaussian process fit is illustrated in \Cref{fig:empirical_sensitivity_stochastic} (Left), stochastic trajectories are presented in \Cref{fig:empirical_sensitivity_stochastic} (Center), and distributions of long-term outcomes are presented in \Cref{fig:empirical_sensitivity_deterministic} (Right).
When there is no sampling noise (i.e., $\eta=0$), group $B$ dominates iff $p^0_B>0.5$ as expected. With the natural variance of the task (i.e., using the inferred Gaussian process covariance and $\eta=1$), results are relatively robust, and group $B$ eventually dominates with high probability for all $p^0_B>0.5$. When variance is excessively increased (i.e., $\eta=5$), outcomes also become noisy, and dynamics can converge to group $A$ with some probability even for $p^0_B\approx 0.55$.

\subsection{Time to Convergence} 
\label{app:sensitivity_time_to_convergence}
For unstable equilibria, a key question is how fast the system transitions to a stable equilibrium state, and what affects this rate. Here we measure the time to reach (approximate) domination by one group as a function of the initial state.

\paragraph{Method.} We extend the analysis of the CIFAR-10 experiment, described in \Cref{subsec:cifar_experiment}. To obtain a smooth population game, we use the mean curves of the Gaussian process described in \Cref{app:sampling_noise_sensitivity_analysis}. For each initial condition, we simulate the replicator dynamics, and record the number of steps until one group is approximately dominant, formally $\min \vp^t \le \ParamSensitivityTimeToDominanceThreshold$.

\paragraph{Results.}  \Cref{fig:empirical_sensitivity_deterministic} (Center) presents the measured convergence times. Results show that the time to convergence is roughly linear in $p^0_B$ for $p^0_B\in[0.1,0.4]\cup[0.6,0.9]$, but tends to infinity as the initial state approaches the uniform distribution ($p^0_B\to 0.5)$. The implication is that for initially balanced populations, natural selection forces are relatively weak, and and higher-order effects (such as sampling noise, or other exogenous forces) can dominate the dynamics.

\begin{figure*}
    \centering
    \includegraphics[width=\textwidth]{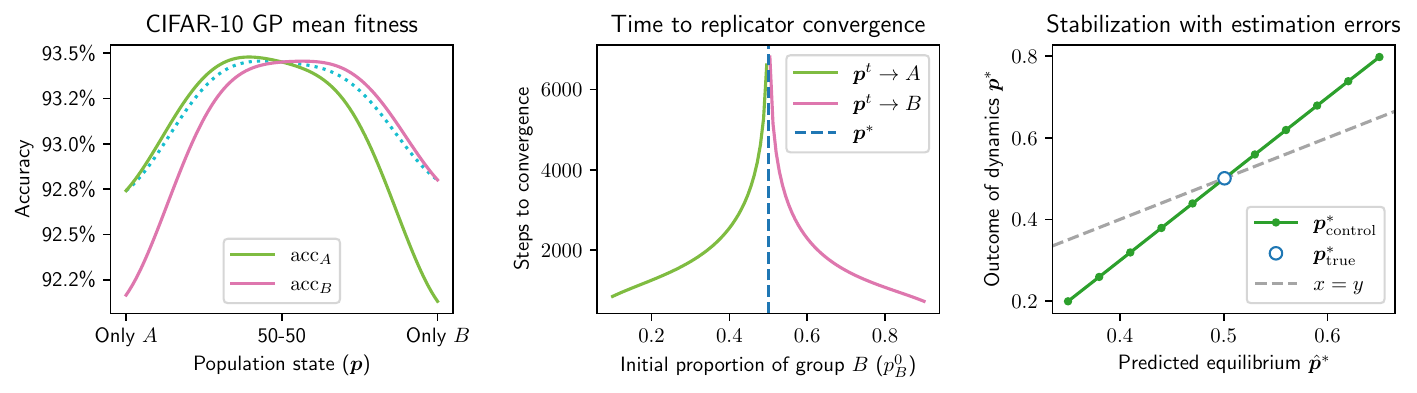}
    \vspace{-1em}
    \caption{
    Empirical analysis of deterministic sensitivity aspects (Appendices \ref{app:sensitivity_time_to_convergence}, \ref {app:sensitivity_steering_estimate}).
    \textbf{(Left)}
    Evolutionary prediction game induced by CIFAR-10 with horizontal flips and Resnet-9. Curves are smoothed by fitting a Gaussian process, as described in \Cref{app:sampling_noise_sensitivity_analysis}.
    \textbf{(Center)}
    Time to convergence for different initial conditions. Curve is relatively linear far from the unstable equilibrium, but approach infinity as the initial condition approaches the unstable equilibrium.
    \textbf{(Right)}
    Sensitivity to equilibrium estimation errors in steering, showing a linear relation.
    }
\label{fig:empirical_sensitivity_deterministic}
\end{figure*}

\subsection{Sensitivity Analysis for Stabilization} 
\label{app:sensitivity_steering_estimate}

Our proposed stabilization mechanism (\Cref{subsec:stabilization}) assumes access to the true equilibrium $\vp^*$, which may not be known in practice. Although our experiments in \Cref{subsec:cifar_experiment} use an estimated $\vp^*$ and results comply with our theoretical findings, they do reveal how robust outcomes are to misspecification of $\vp^*$. For this, we present a sensitivity analysis quantifying the deviation of the reached state from the estimated equilibrium $\vp^*$. 

\paragraph{Method.} We simulate the stabilization method presented in \Cref{subsec:stabilization} for varying population states around the uniform population state equilibrium, and measure the resulting long-term outcomes of the dynamics.

\paragraph{Results.} Results are presented in \Cref{fig:empirical_sensitivity_deterministic} (Right), and show a linear relation between the estimated equilibrium and the eventual outcome of the stabilized dynamics. An additional observation is that misspecification affects only the population composition, and not welfare. This is because the algorithm ensures that the returned classifier acts ``as if'' the system is under equilibrium, in which accuracy for both groups is the same.

\subsection{Three-Group Dynamics for Different Learning Algorithms}
\label{app:fairness_additional_experiments}

To further establish the results of \Cref{subsec:fairness_empirical}, we compare three-group dynamics induced by different algorithms in the same setting. 
\paragraph{Method.} We follow the same procedure presented in \Cref{subsec:fairness_empirical}, and compute the dynamics for a linear SVM (from scikit-learn), logistic regression (scikit-learn), and an XGBoost classifier, all with default regularization parameters. Replicator dynamics are simulated using \Cref{eq:discrete_replicator}, and basins of attraction are computed by simulating the dynamics until convergence. The initial condition for all algorithms is $
\vp^0=\left(
p_\mathrm{(TX)},
p_\mathrm{(NY)},
p_\mathrm{(CA)}
\right)
= (0.375, 0.5  , 0.125)
$.

\paragraph{Results.} Results are presented in \Cref{fig:foktables_algs}. For identical initial conditions, each learning algorithm leads to a qualitatively different outcome: dynamics induced by the linear SVM converge to dominance of the Texas (TX) user group, dynamics induced by logistic regression converge towards a coexistence saddle point, and XGBoost dynamics converge towards dominance of the California (CA) user group. In particular, for the logistic regression system, note that the population state evolves along the attraction basin boundary (separatrix). By the definition of the replicator dynamics, states along the attraction basin boundary satisfy $\acc_\mathrm{(CA)}(h_\vp)=\acc_\mathrm{(TX)}(h_\vp)$, and therefore exhibit a two-phase phenomenon similar to the one presented in \Cref{subsec:fairness_empirical} (an early balance followed by competition between the remaining groups). In contrast, all interior points in the XGBoost dynamics evolove toward CA domination, precluding the existence of the two-phase phenomenon in this setting.

\begin{figure*}
    \centering
    \includegraphics[width=\textwidth]{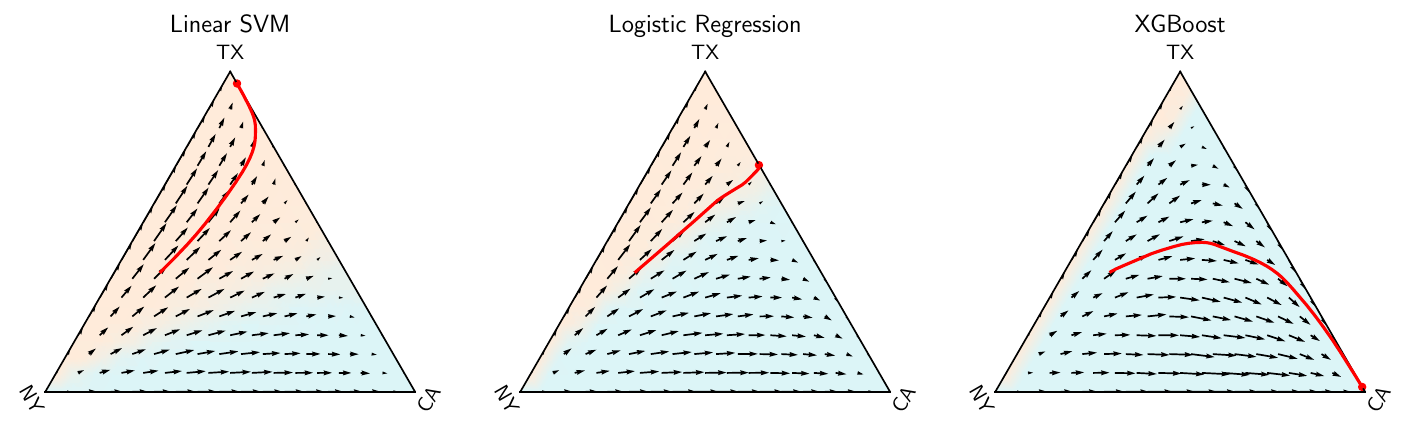}
    \vspace{-1em}
    \caption{
    Comparing three-group replicator dynamics across different learning algorithms for the ACSIncome prediction task (\Cref{app:fairness_additional_experiments}). Background colors represent basins of attraction, and red lines represent trajectories starting from the same initial condition. 
    }
\label{fig:foktables_algs}
\end{figure*}

\end{document}